%% file: paper.tex
\documentclass[11pt]{article}

\usepackage[colorlinks=true,linkcolor=blue,urlcolor=black,citecolor=blue,breaklinks]{hyperref}
\usepackage{color}
\usepackage{enumerate}
\usepackage{graphicx}
\usepackage[nice]{nicefrac}
\usepackage{algorithm, algcompatible}
\usepackage{varwidth}%
\usepackage{mathrsfs}
\usepackage{mathtools}
\usepackage[labelfont=bf,font={small}]{caption}
\usepackage[labelfont=bf]{subcaption}
\usepackage{wrapfig}
\usepackage{overpic}
\usepackage{multirow}
\usepackage[numbers,sort, compress]{natbib}

\usepackage{fullpage}
\usepackage{authblk}

\usepackage{amsmath,amsthm,amssymb,colonequals,etoolbox}
\usepackage{thmtools}
\usepackage{thm-restate}

\input{commands}

\author[1]{Stephen Tu}
\author[2]{Alexander Robey}
\author[1]{Tingnan Zhang}
\author[1,2]{Nikolai Matni}
\affil[1]{Google Brain Robotics}
\affil[2]{Department of Electrical and Systems Engineering, University of Pennsylvania}

\date{June 8, 2021, Revised: \today}

\title{On the Sample Complexity of Stability Constrained Imitation Learning}

\begin{document}

\maketitle

\begin{abstract}
\input{abstract}
\end{abstract}

\input{intro}
\input{problem_formulation}
\input{incremental-gain-stability}
\input{algorithms}
\input{experiments}

\input{conclusion}

\section*{Acknowledgements}
The authors would like to thank Vikas Sindhwani, Sumeet Singh,
Andy Zeng, and Lisa Zhao for their valuable comments and suggestions.
NM is generously supported by NSF award CPS-2038873, NSF CAREER award ECCS-2045834, and a Google Research Scholar Award.

\bibliographystyle{unsrtnat}
\bibliography{paper}

\appendix

\newpage
\tableofcontents
\newpage

\input{appendix/experiments-stability}
\input{appendix/experiments-laikago}
\input{appendix/incremental-gain-stability}
\input{appendix/examples}
\input{appendix/main-proof}

\end{document}

%% file: commands.tex
\usepackage{amsmath}
\usepackage{amssymb}
\usepackage{amsthm}
\usepackage{tikz}
\usepackage{pgfplots}
\usepackage{graphicx}
\usepackage{epstopdf}
\usepackage[labelfont=bf,font=footnotesize]{caption}
\usepackage[labelfont=bf,font=footnotesize]{subcaption}
\usepackage{xcolor}
\usepackage{cite}
\usepackage{algorithmicx}
\usepackage{algpseudocode}
\usepackage{algorithm}
\usepackage{bbm}
\usepackage{makecell}
\usepackage{authblk}
\usepackage[shortlabels]{enumitem}
\usepackage{listings}
\usepackage{thmtools}

\newtheorem{thm}{Theorem}[section]
\newtheorem{mydef}[thm]{Definition}
\newtheorem{myprop}[thm]{Proposition}

\newtheorem{mythm}[thm]{Theorem}
\newtheorem{myassump}[thm]{Assumption}

\DeclareMathOperator*{\argmin}{arg\!\min}

\DeclareMathOperator*{\sgn}{sgn}

\DeclareMathOperator*{\diag}{diag}

\DeclareMathOperator*{\conv}{conv}

\newcommand{\R}{\ensuremath{\mathbb{R}}}

\newcommand{\N}{\ensuremath{\mathbb{N}}}

\newcommand{\norm}[1]{\lVert #1 \rVert}
\newcommand{\bignorm}[1]{\left\lVert #1 \right\rVert}

\newcommand{\ip}[2]{\ensuremath{\langle #1, #2 \rangle}}

\newcommand{\E}{\mathbb{E}}
\newcommand{\abs}[1]{\ensuremath{| #1 |}}
\newcommand{\bigabs}[1]{\ensuremath{\left| #1 \right|}}

\newcommand{\ind}{\mathbf{1}}

\renewcommand{\Pr}{\mathbb{P}}
\newcommand{\T}{\mathsf{T}}

\newcommand{\calC}{\mathcal{C}}
\newcommand{\calD}{\mathcal{D}}

\newcommand{\calE}{\mathcal{E}}

\newcommand{\calR}{\mathcal{R}}

\newcommand{\calS}{\mathcal{S}}
\newcommand{\calK}{\mathcal{K}}

\DeclareMathOperator*{\esssup}{ess\,sup}

\numberwithin{equation}{section}

\renewcommand{\geq}{\geqslant}

\renewcommand{\leq}{\leqslant}
\renewcommand{\preceq}{\preccurlyeq}

\newcommand{\opnorm}[1]{\norm{#1}_{\mathrm{op}}}

\newcommand{\fcl}{f_{\mathrm{cl}}}
\newcommand{\minimize}{\mathrm{minimize}}

\newcommand{\rmd}{\mathrm{d}}
\newcommand{\e}{\varepsilon}

\newcommand{\flow}{x}

%% file: abstract.tex
We study the following question in the context of imitation learning for continuous control: how are the underlying stability properties of an expert policy reflected in the sample-complexity of an imitation learning task?  We provide the first results showing that a surprisingly granular connection can be made between the underlying expert system's \emph{incremental gain stability}, a novel measure of robust convergence between pairs of system trajectories, and the dependency on the task horizon $T$ of the resulting generalization bounds.  In particular, we propose and analyze \emph{incremental gain stability constrained} versions of behavior cloning and a DAgger-like algorithm, and show that the resulting sample-complexity bounds naturally reflect the underlying stability properties of the expert system.  As a special case, we delineate a class of systems for which the number of trajectories needed to achieve $\varepsilon$-suboptimality is \emph{sublinear} in the task horizon $T$, and do so without requiring (strong) convexity of the loss function in the policy parameters.  Finally, we conduct numerical experiments demonstrating the validity of our insights on both a simple nonlinear system for which the underlying stability properties can be easily tuned, and on a high-dimensional quadrupedal robotic simulation.

%% file: intro.tex
\section{Introduction}
\label{sec:intro}

Imitation Learning (IL) techniques \citep{hussein2017imitation,osa2018algorithmic} use demonstrations of desired behavior, provided by an expert, to train a policy. IL offers many appealing advantages: it is often more sample-efficient than reinforcement learning~\citep{sun2017deeply,ross2011reduction}, and can lead to policies that are more computationally efficient to evaluate online~\citep{hertneck2018learning,yin2020imitation} than optimization-based experts.  Indeed, there is a rich body of work demonstrating the advantages of IL-based methods in a range of applications including video-game playing~\citep{ross2010efficient,ross2011reduction}, humanoid robotics~\citep{schaal1999imitation}, and self-driving cars~\citep{codevilla2018end}. 
Safe IL further seeks to provide guarantees on the stability or safety properties of policies produced by IL. 
Methods drawing on tools from Bayesian deep learning~\citep{lee2018safe,menda2017dropoutdagger,menda2018ensembledagger}, PAC-Bayes~\citep{ren2020generalization}, stability regularization~\citep{sindhwani2018learning}, or robust control~\citep{yin2020imitation,hertneck2018learning}, are able to provide varying levels of guarantees in the context of IL.

However, when applied to continuous control problems, little to no insight is given into how the underlying stability properties of the expert policy affect the sample-complexity of the resulting IL task.
In this paper, we address this gap and answer the question: what makes an expert policy easy to learn?  
Our main insight is that when an expert policy satisfies a suitable quantitative notion of robust \emph{incremental} stability, i.e., when pairs of system trajectories under the expert policy robustly converge towards each other, and when learned policies are also constrained to satisfy this property, then IL can be made provably efficient.  We formalize this insight through the notion of \emph{incremental gain stability} constrained IL algorithms, and in doing so, quantify and generalize previous observations of efficient and robust learning subject to contraction based stability constraints.

\input{related}

\paragraph{Contributions.} To provide fine-grained insights into the relationship between system stability and sample-complexity, we first define and analyze the notion of incremental gain stability (IGS) for a nonlinear dynamical system.  IGS provides a quantitative measure of robust convergence between system trajectories, that in our context strictly expands on the guarantees provided by contraction theory~\citep{lohmiller1998contraction} by allowing for a graceful degradation away from exponential convergence rates.

We then propose and analyze the sample-complexity properties of IGS-constrained imitation learning algorithms, and show that the graceful degradation in stability translates into a corresponding degradation of generalization bounds by linking nonlinear stability and statistical learning theory.  In particular, we show that when imitating an IGS expert policy, IGS-constrained behavior cloning requires $m \gtrsim  q \cdot T^{2a(1-1/a^2)} \cdot \e^{-2a}$
trajectories to achieve imitation loss bounded by $\varepsilon$, where $T$ is the task horizon, $q$ is the effective number of parameters of the function class for the learned policy, and $a \in [1,\infty)$ is an IGS parameter determined by the expert policy.  We show that $a=1$ for contracting systems, leading to \emph{task-horizon independent} bounds scaling as $m \gtrsim {q}/{\varepsilon^{2}}$.
Furthermore, we construct a simple family of systems 
where the IGS parameter $a$ satisfies
$a = 1+p$ for $p\in (0,\infty)$.
This yields sample-complexity that scales as $m \gtrsim q \cdot T^{\nicefrac{2p(2+p)}{1+p}} \cdot \varepsilon^{-2(1+p)}$, which makes clear that an increase/decrease in $p$ yields a corresponding increase/decrease in sample-complexity. 

Motivated by the empirical success and widespread adoption of DAgger and DAgger-like algorithms, we also extend our analysis to an IGS-constrained DAgger-like algorithm.  We show that this algorithm enjoys comparable stability dependent sample-complexity guarantees, requiring $m \gtrsim q \cdot T^{2a^2\left(1-1/a\right)\left(1+1/a+3/2a^2\right)} \cdot \e^{-2a^2}$ trajectories to achieve $\varepsilon$-bounded imitation loss, again recovering time-independent bounds for contracting systems that gracefully degrade when applied to our family of systems satisfying $a=1+p$.  

Together, our results are the first to delineating a class of systems where the sample-complexity bounds for imitation learning
scale sublinearly in the task-horizon $T$, and do so without requiring (strong) convexity of the loss function in the policy parameters.
We conclude by demonstrating the validity of our theoretical results on (a) our simple family of nonlinear systems for which the underlying IGS properties can be quantitatively tuned, and (b) a high-dimensional nonlinear quadrupedal robotic system. Empirically, we find that the sample-complexity scaling predicted by the underlying stability properties of the expert policy are indeed observed in practice.

%% file: related.tex
\paragraph{Related Work.} There exist a rich body of work examining the interplay between stability theory and learning dynamical systems/policies satisfying stability/safety properties from demonstrations.

\emph{Nonlinear stability and learning from demonstrations:} Our work applies tools from nonlinear stability theory to analyze the sample complexity of IL algorithms.  Concepts from nonlinear stability theory, such as Lyapunov stability or contraction theory~\citep{lohmiller1998contraction}, have also been successfully applied to learn autonomous nonlinear dynamical systems satisfying desirable properties such as (incremental) stability or controllability. 
As demonstrated empirically in~\citep{singh2020learning,lemme2014neural,ravichandar2017learning,sindhwani2018learning}, using such stability-based regularizers to trim the hypothesis space results in more data-efficient and robust learning algorithms. However, no quantitative sample-complexity bounds are provided. 

\emph{IL under covariate shift:} Vanilla IL (e.g., Behavior Cloning) is known to be sensitive to covariate shift: as soon as the learned policy deviates from the expert policy, errors begin to compound, leading the system to drift to new and possibly dangerous states \citep{pomerleau1989alvinn,ross2011reduction}.  Representative IL algorithms that address this issue include DAgger \citep{ross2011reduction} (an on-policy approach) and DART \citep{laskey2017dart} (an off-policy approach).  Both approaches seek to mitigate the effects of system drift at test-time by augmenting the data-set created by the expert: DAgger iteratively augments its data-set of trajectories with appropriately labeled and/or corrected data of the previous policy, whereas DART injects noise into the supervisor demonstrations and allows the supervisor to provide corrections as needed.
For loss functions that are strongly convex in the policy parameters, DAgger enjoys $\tilde O(T)$ sample-complexity in the task horizon $T$, and this bound degrades to $\tilde O(T^2)$ when loss functions are only convex. On the other hand, we are not aware of finite-data guarantees for DART. 
IL algorithms more explicitly focused on stability/safety leverage tools from Bayesian deep learning~\citep{menda2017dropoutdagger,menda2018ensembledagger},  model-predictive-control~\citep{hertneck2018learning}, robust control~\citep{yin2020imitation}, and PAC-Bayes~\citep{ren2020generalization}. 
While the approach, generality, and strength of guarantees provided by the aforementioned works vary, none provide insight as to how the stability properties of the expert affect the sample complexity of the corresponding IL task.

%% file: problem_formulation.tex
\section{Problem Statement}
\label{sec:problem_statement}

We consider the following discrete-time  dynamical system:
\begin{align}
 x_{t+1} = f(x_t, u_t), \quad x_t \in \R^n, \:\: u_t \in \R^d. \label{eq:control_affine}
\end{align}
Let $\flow_t(\xi, \{u_t\}_{t \geq 0})$ denote the state $x_t$ of the 
dynamics \eqref{eq:control_affine} with input signal $\{u_t\}_{t \geq 0}$
and initial condition $x_0 = \xi$.
For a policy $\pi : \R^n \rightarrow \R^d$, let $\flow_t^\pi(\xi)$
denote $x_t$ when $u_t = \pi(x_t)$.
Let $X \subseteq \R^{n}$ be a compact set and let $T \in \N_+$ be the time-horizon over which 
we are interested in the behavior of \eqref{eq:control_affine}.  
We generate trajectories by drawing random initial conditions
from a distribution $\calD$ over $X$.

We fix an expert policy $\pi_\star : \R^n \rightarrow \R^d$
which we wish to imitate.
The quality of our imitation is measured through 
the following \emph{imitation loss}:
\begin{align}
    \ell_{\pi'}(\xi;\pi_1,\pi_2) := \sum_{t=0}^{T-1} \norm{\Delta_{\pi_1,\pi_2}(\flow_t^{\pi'}(\xi))}_2, \quad \Delta_{\pi_1,\pi_2}(x) := \pi_1(x) - \pi_2(x).
\end{align}
The imitation loss function $\ell_{\pi'}(\xi;\pi_1,\pi_2)$ keeps a running tally of the difference $\Delta_{\pi_1,\pi_2}(\flow_t^{\pi'}(\xi))$ of how actions taken by policies $\pi_1$ and $\pi_2$ enter the system \eqref{eq:control_affine} when the system is evolving under policy $\pi'$ starting from initial condition $x_0=\xi$.

We now formally state the problem considered in this paper. Fix a known
system \eqref{eq:control_affine}, and pick a
tolerance $\varepsilon >0 $ and failure probability $\delta\in (0,1)$. Our goal is to design and analyze imitation learning algorithms that produce a policy $\hat{\pi}$ using $m = m(\varepsilon, \delta)$ trajectories of length $T$ seeded from initial conditions $\{\xi_i\}_{i=1}^m \sim{} \mathcal{D}^m$, such that with probability at least $1-\delta$, the learned policy $\hat\pi$ induces a state/input trajectory distribution that satisfies
$\E_{\xi \sim \calD} \ell_{\hat\pi}(\xi;\hat\pi, \pi_\star) \leq \varepsilon$.
Crucially, we seek to understand how the underlying stability properties of the expert policy $\pi_\star$ manifest themselves in the number of 
required trajectories $m(\varepsilon,\delta)$.

We note that bounding the imitation loss 
has immediate implications on the e.g., safety, stability, and
performance of the learned policy $\hat{\pi}$.
Concretely, let 
$h : \R^{n \times T} \rightarrow \R^s$
denote an $L_h$-Lipschitz observable function of a trajectory: examples of valid observable functions include Lyapunov and barrier inequalities and semi-algebraic constraints on the state or state-feedback policy. 
Then,
\begin{align}
    \E_{\xi\sim\calD} \norm{ h_{\pi_\star}(\xi) - h_{\hat{\pi}}(\xi) }_2 \leq L_h\E_{\xi\sim\calD} {\sum_{t=1}^{T}} \norm{\flow_t^{\pi_\star}(\xi) - \flow_t^{\hat{\pi}}(\xi) }_2, \label{eq:observable_upper_bound}
\end{align}
where $h_\pi(\xi) := h( \{\flow_t^\pi(\xi)\}_{t=0}^{T} )$.
We will subsequently see how
the term $\sum_{t=1}^{T} \norm{\flow_t^{\pi_\star}(\xi) - \flow_t^{\hat{\pi}}(\xi) }_2$ can be upper bounded
by the imitation loss $\ell_{\hat{\pi}}(\xi;\hat{\pi},\pi_\star)$,
and thus bounds on the imitation loss imply bounds on the
deviations between the observables $h_{\pi_\star}$ and $h_{\hat{\pi}}$.

%% file: incremental-gain-stability.tex
\section{Incremental Gain Stability}
\label{sec:inc_gain_stability}

The crux of our analysis relies on a property which we call 
\emph{incremental gain stability (IGS)}. Before formally defining IGS, we motivate the need for a quantitative characterization of convergence rates between system trajectories.
A key quantity that repeatedly appears in our analysis is the following
sum of trajectory discrepancy induced by policies $\pi_1$ and $\pi_2$:
\begin{align}
    \mathsf{disc}_T(\xi; \pi_1, \pi_2) := {\sum_{t=1}^{T}} \norm{\flow_t^{\pi_1}(\xi) - \flow_t^{\pi_2}(\xi)}_2. \label{eq:delta_traj}
\end{align}
We already saw this quantity appear naturally in \eqref{eq:observable_upper_bound}.
Furthermore, we will reduce analyzing the performance of behavior cloning and our DAgger-like algorithm to bounding the discrepancy \eqref{eq:delta_traj} between trajectories induced by the expert policy $\pi_\star$ and a learned policy $\hat\pi$.

The simplest way to bound \eqref{eq:delta_traj} is to use a discrete-time
version of Gr{\"{o}}nwall's inequality: if the map $f$ defining
\eqref{eq:control_affine} is $L_f$-Lipschitz, 
in addition to the policies $\pi_1$ and $\pi_2$ 
being $L_\pi$-Lipschitz, then (assuming $L_f L_\pi \gg 1$) we can upper bound the discrepancy \eqref{eq:delta_traj} by:
\begin{align}
    \mathsf{disc}_T(\xi; \pi_1,\pi_2) \lesssim (L_f L_\pi)^T \cdot \ell_{\pi_1}(\xi; \pi_1, \pi_2). \label{eq:gronwall_estimate}
\end{align}
This bound formalizes the intuition that the discrepancy \eqref{eq:delta_traj}
scales in proportion to the deviation between policies $\pi_1$ and $\pi_2$, summed
along the trajectory.
Unfortunately, this bound is undesirable due to its exponential dependence on
the horizon $T$. In order to improve the dependence on $T$, we need to 
assume some stability properties on the dynamics $f$.
We start by drawing inspiration from the definition of
\emph{incremental input-to-state stability}~\citep{tran16incrementaliss}.\footnote{
Definition~\ref{def:inc_iss} is more general than \citet[Definition 6]{tran16incrementaliss}
in that we only require a bound with respect to an input perturbation of one of the trajectories,
not both.}
This definition relies on standard comparison class definitions,
which we briefly review. A class $\mathcal{K}$ function
$\sigma : \R_{\geq 0} \rightarrow \R_{\geq 0}$
is continuous, increasing, and satisfies $\sigma(0) = 0$.
A class $\mathcal{K}_\infty$ function is class $\mathcal{K}$ and unbounded.
Finally, a class $\mathcal{KL}$ function $\beta : \R_{\geq 0} \times \R_{\geq 0} \rightarrow \R_{\geq 0}$ satisfies (i) $\beta(\cdot, t)$ is class $\mathcal{K}$ for every $t$ and (ii)
$\beta(s, \cdot)$ is continuous, decreasing, and tends to zero for every $s$.
\begin{mydef}[Incremental input-to-state-stability ($\delta$ISS)]
\label{def:inc_iss}
Consider the discrete-time dynamics $x_{t+1} = f(x_t, u_t)$, and
let $\flow_t(\xi, \{u_t\}_{t \geq 0})$ denote the state $x_t$
initialized from $x_0 = \xi$ with input signal $\{u_t\}_{t \geq 0}$.
The dynamics $f$ is said to be \emph{incremental input-to-state-stable} if
there exists a class $\mathcal{KL}$ function $\zeta$ and
class $\calK_\infty$ function $\gamma$ such that
for every $\xi_1, \xi_2 \in X$, $\{u_t\}_{t \geq 0} \subseteq U$, and $t \in \N$,
\begin{align*}
    \norm{\flow_t(\xi_1, \{u_t\}_{t \geq 0}) - \flow_t(\xi_2, \{0\}_{t \geq 0})}_X \leq \zeta(\norm{\xi_1-\xi_2}_X, t) + \gamma\left( \max_{0 \leq k \leq t-1} \norm{u_k}_U\right).
\end{align*}
\end{mydef}
Definition~\ref{def:inc_iss} improves the Gr{\"{o}}nwall-type estimate \eqref{eq:gronwall_estimate}
in the following way. Suppose the closed-loop system defined by
$\tilde{f}(x, u) := f(x, \pi_2(x) + u)$ is $\delta$ISS. Then the algebraic identity
\begin{align*}
    f(x, \pi_1(x)) = f(x, \pi_2(x) + \pi_1(x) - \pi_2(x)) = \tilde{f}(x, \Delta_{\pi_1,\pi_2}(x))
\end{align*}
allows us to treat $\{\Delta_{\pi_1,\pi_2}(\flow_t^{\pi_1}(\xi))\}_{t \geq 0}$ as an input signal, yielding
\begin{align*}
    \mathsf{disc}_T(\xi;\pi_1,\pi_2)\leq T \cdot \gamma\left( \ell_{\pi_1}(\xi; \pi_1, \pi_2) \right).
\end{align*}
This bound certainly improves the dependence on $T$, but is not sharp:
for stable linear systems,
it is not hard to show that
$\mathsf{disc}_T(\xi;\pi_1,\pi_2)\leq O(1) \cdot \ell_{\pi_1}(\xi; \pi_1, \pi_2)$.
In order to capture sharper rate dependence on $T$, we need to 
modify the definition to more explicitly quantify convergence rates.
\begin{mydef}[Incremental gain stability]
\label{def:inc_gain_stable_poly}
Consider the discrete time dynamics $x_{t+1} = f(x_t, u_t)$.
Let $a, b, \zeta, \gamma \in [1, \infty)$ and $\alpha_0$ be positive.
Put $\Psi := (\alpha_0, \zeta, \gamma)$.
We say that $f$ is $(a,b, \Psi)$\emph{-incrementally-gain-stable} (abbreviated as $(a,b,\Psi)$-IGS) if
for all horizon lengths $T \in \N$, initial conditions $\xi_1, \xi_2 \in X$,
and input sequences $\{u_t\}_{t \geq 0} \subseteq U$, we have:
\begin{align}\label{eq:inc_gain_stability_defn}
    {\sum_{t=0}^{T}}\min\{ \norm{\Delta_t}_X, \norm{\Delta_t}_X^a \} \leq \zeta \norm{\xi_1 - \xi_2}_X^{\alpha_0} + \gamma {\sum_{t=0}^{T-1}} \max\{\norm{u_t}_U, \norm{u_t}_U^b\}.
\end{align}
Here, 
$\Delta_t := \flow_t(\xi_1, \{u_t\}_{t \geq 0}) - \flow_t(\xi_2, \{0\}_{t \geq 0})$.
\end{mydef}

IGS quantitatively bounds the amplification of an input signal $\{u_t\}_{t\geq 0}$ (and differences in initial conditions $\xi_1,\xi_2$) on the corresponding system trajectory discrepancies  $\{\Delta_t\}_{t\geq 0}$.  
Note that a system that is incrementally gain stable is automatically
$\delta$ISS. 
IGS also captures the phase transition that occurs in non-contracting systems about the unit circle.  For example, when $\|\Delta_t\|\leq 1$ and $\|u_t\|\leq 1$ for all $t\geq 0$, inequality \eqref{eq:inc_gain_stability_defn} reduces to $\sum_{t=0}^{T} \norm{\Delta_t}_X^{a} \leq \zeta \norm{\xi_1 - \xi_2}_X^{\alpha_0} + \gamma \sum_{t=0}^{T-1} \norm{u_t}_U.$  Finally, as IGS measures signal-to-signal ($\{u_t\}_{t \geq 0} \to \{\Delta_t\}_{t \geq 0}$) amplification, it is well suited to analyzing learning algorithms operating on system trajectories.

For a policy $\pi$, we define $\fcl^\pi(x, u) := f(x, \pi(x) + u)$.
We show that (cf.~Proposition~\ref{prop:inc_gain_stab_compare_inputs}) 
if $\fcl^{\pi_2}$ is $(a,b,\Psi)$-IGS, then
\begin{align}\label{eq:inc_gain_disc}
    \mathsf{disc}_T(\xi; \pi_1, \pi_2) \leq 4 \gamma T^{1-1/a} \cdot \max\left\{
    \ell_{\pi_1}(\xi; \pi_1, \pi_2)^{1/a}, 
    \ell_{\pi_1}(\xi; \pi_1, \pi_2)^{b}
    \right\}.
\end{align}
With this bound, the dependence on $T$ is allowed to interpolate between $1$ and $T$,
and the dependence on $\ell_{\pi_1}(\xi;\pi_1,\pi_2)$ is made explicit.
Next, we state a Lyapunov based sufficient condition for
Definition~\ref{def:inc_gain_stable_poly} to hold. Unlike Definition~\ref{def:inc_gain_stable_poly},
the Lyapunov condition is checked pointwise in space rather than over entire trajectories.
\begin{myprop}[Incremental Lyapunov function implies stability]
\label{prop:lyap_characterization}
Let $a, b \in [1, \infty)$, $\alpha_0 \in [1, a]$, 
and $\underline{\alpha}, \overline{\alpha}, \mathfrak{a}, \mathfrak{b}$ be positive finite constants
satisfying $\underline{\alpha} \leq \overline{\alpha}$.
Suppose there exists a non-negative function $V : \R^{n} \times \R^{n} \rightarrow \R_+$ such that for all $x, y \in \R^n$ and $u \in U$,
\begin{align}
   \underline{\alpha} \norm{x-y}_X^{\alpha_0} &\leq V(x, y) \leq \overline{\alpha} \norm{x-y}_X^{\alpha_0}, \label{eq:inc_gain_stable_lyap_func} \\
    V(f(x, u), f(y, 0)) - V(x, y) &\leq - \mathfrak{a} \min\{\norm{x-y}_X, \norm{x-y}_X^{a}\} + \mathfrak{b} \max\{\norm{u}_U, \norm{u}_U^{b}\}. \label{eq:inc_gain_stable_lyap_cond}
\end{align}
Then, $f$ is $(a, b, \Psi)$-incrementally-gain-stable with
$\Psi = \left(\alpha_0, \frac{\overline{\alpha}}{\underline{\alpha} \wedge \mathfrak{a}}, \frac{\mathfrak{b}}{\underline{\alpha} \wedge \mathfrak{a}}\right)$.
\end{myprop}

\subsection{Examples of Incremental Gain Stability}
\label{sec:inc_gain_stability:examples}

Our first example of incremental gain stability is a
contracting system~\citep{lohmiller1998contraction}.
\begin{restatable}{myprop}{incgaincontraction}
\label{prop:inc_gain_contraction}
Consider the dynamics $x_{t+1} = f(x_t, u_t)$. Suppose that $f$ is \emph{autonomously contracting},
i.e., there exists a positive definite metric $M(x)$ and a scalar $\rho \in (0, 1)$ such that:
\begin{align*}
    \frac{\partial f}{\partial x}(x, 0)^\T M(f(x, 0)) \frac{\partial f}{\partial x}(x, 0) \preceq \rho M(x) \quad \forall x \in \R^n.
\end{align*}
Suppose also that the metric $M$ 
satisfies $\underline{\mu} I \preceq M(x) \preceq \overline{\mu} I$ for all $x \in \R^n$,
and that there exists a finite $L_u$ such that the dynamics satisfies $\norm{f(x, u) - f(x, 0)}_2 \leq L_u \norm{u}_2$ for all $x \in \R^n$, $u \in \R^d$.  
Then we have that $f$ is $(1, 1, \Psi)$-IGS,
with
$\Psi = \left(1, \sqrt{\frac{\overline{\mu}}{\underline{\mu}}} \frac{1}{1-\sqrt{\rho}}, L_u \sqrt{\frac{\overline{\mu}}{\underline{\mu}}} \frac{1}{1-\sqrt{\rho}} \right)$.
\end{restatable}
Three concrete examples of autonomously contracting systems include:
\begin{enumerate}[(a)]
    \item Piecewise linear systems $f(x, u) = \left( \sum_{i=1}^{K} A_i \ind\{ x \in \calC_i \} \right)x + B u$
where the $A_i$'s are stable, $\{\calC_i\}$ partitions $\R^n$, and there exists a common quadratic Lyapunov function $V(x) = x^\T P x$
which yields the metric $M(x) = P$,
    \item $f(x, u) = \log(1 + x^2) + u$ with the metric $M(x) = 2 [ 1 + \exp(-\abs{x})]^{-1}$, and
    \item $f(x, u) = x - \eta \left[ \nabla V(x) + u \right]$
where $V(x)$ is a twice differentiable potential function
satisfying $\mu I \preceq \nabla^2 V(x) \preceq L I$,
and $0 < \eta \leq 1/L$ \citep{wensing20beyondconvexity}.
\end{enumerate}
Importantly, contracting systems enjoy time-independent discrepancy measures since \eqref{eq:inc_gain_disc} reduces to $O(\ell_{\pi_1}(\xi;\pi_1,\pi_2))$, i.e., contracting nonlinear systems behave like stable linear systems, up to contraction metric dependent constants.

Our next example illustrates a family of systems that degrade away from
exponential rates.
\begin{myprop}
\label{prop:inc_gain_p}
Consider the scalar dynamics $x_{t+1} = x_t - \eta x_t \frac{\abs{x_t}^p}{1 + \abs{x_t}^p} + \eta u_t$ for $p \in (0, \infty)$. Then as long as $0 < \eta < \frac{4}{5+p}$, we have that $f$ is $(1+p, 1, \Psi)$-IGS, with
$\Psi = \left( 1, \frac{2^{2+p}}{\eta}, 2^{2+p} \right)$
\end{myprop}
The system described in Proposition~\ref{prop:inc_gain_p} behaves like a stable linear system when $|x_t|\geq 1$, and like a polynomial system when $|x_t|< 1$ (hence $a = 1+p$).  This example 
highlights the need to be able to capture a phase-transition within our definitions and Lyapunov characterizations.

%% file: algorithms.tex
\section{Algorithms and Theoretical Results}
\label{sec:algorithms}

In this section we define and analyze IGS-constrained imitation learning algorithms.  We begin by 
introducing our main assumption of dynamics and policy class regularity.

\begin{myassump}[Regularity]
\label{assumption:main}
We assume that the dynamics $f$, policy class $\Pi$, expert $\pi_\star$,
and initial condition distribution $\calD$ satisfy:
\begin{enumerate}[(a), noitemsep]
    \item The dynamics map $f$ satisfies $f(0, 0) = 0$.
    \item The policy class $\Pi$ is convex
    and $\pi(0) = 0$ for all $\pi \in \Pi$.
    \item The distribution $\calD$ over initial conditions satisfies $\norm{\xi}_2 \leq B_0$ a.s.\ for $\xi \sim \calD$.
    \item The expert policy $\pi_\star \in \Pi$.
    \item $\Delta_{\pi_1,\pi_2}$ is $L_\Delta$-Lipschitz for all $\pi_1, \pi_2 \in \Pi$.
    \item The constants $B_0,L_\Delta \in [1, \infty)$.
\end{enumerate}
\end{myassump}

Before turning to our main stability assumption
we briefly remark on Assumption~\ref{assumption:main}(b), which
requires that the policy class $\Pi$ is convex.
This assumption is stronger than is actually necessary.
Instead, we could consider optimizing at epoch $k$
over a function class $\Pi_k$
defined recursively as $\Pi_{k} = \{ \alpha \pi + (1-\alpha) \pi_{k-1} \mid \alpha \in [0, 1], \, \pi \in \Pi, \, \pi_{k-1} \in \Pi_{k-1}  \}$,
with the base case $\Pi_0 = \Pi$. Instead, we choose to
make the assumption that $\Pi$ is convex to simplify the presentation, noting that using the recursive representation
would yield the same sample complexity bounds.\footnote{Our results are derived by bounding
the Rademacher complexity of a particular function class, 
which in our setting is preserved under convex hulls. See Proposition~\ref{prop:rademacher_convex_hull} in the appendix for more details.}

We now turn to our main
stability assumption.
Recall that $\fcl^\pi(x, u) =f(x, \pi(x) + u)$ denotes the closed-loop dynamics induced by a policy $\pi$. Our main stability assumption is that $\fcl^{\pi_\star}$ is $(a, 1, \Psi)$-IGS. 
By only allowing the parameter $a$ to vary,
we are able to capture the stability trade-offs
we are after, while streamline the results and simplifying the 
proofs. Generalizing to the case $b > 1$ is
straightforward, but involves more cumbersome expressions.

\begin{myassump}[Incremental Gain Stability]
\label{assumption:stability}
Let $a, \zeta, \gamma \in [1, \infty)$, and $\alpha_0 > 0$.
Put $\Psi = (\alpha_0, \zeta, \gamma)$, and 
let $\calS(a, b, \Psi)$ denote the set of policies $\pi$ such that $\fcl^{\pi}$
is $(a, b, \Psi)$-IGS.
We assume that $\pi_\star \in \calS(a, 1, \Psi)$.
\end{myassump}

We remark that we assume that the expert policy lies in our policy class, i.e., $\pi_\star\in\Pi$, to guarantee that zero imitation loss can be achieved
in the limit of infinite data; it is straightforward to relax this
assumption to $\Pi \cap \calS(a, 1, \Psi) \neq \emptyset$ and prove results with respect to the best stabilizing policy in class.

\begin{algorithm}[htb]
    \centering
    \begin{algorithmic}[1]
        \Statex \textbf{Input: } Total trajectory budget $m$, number of epochs $E$ that divides $m$, mixing rate $\alpha \in (0,1]$, initial conditions $\left\{\{\xi_i^{k}\}_{i=1}^{m/E}\right\}_{k=0}^{E-1}\sim{}\calD^n$, expert policy $\pi_\star$, stability parameters $(a,b,\Psi)$, and non-negative scalars $\{c_i\}_{i=1}^{E-2}$.
        \State $\pi_0 \gets \pi_\star$, $c_0 \gets 0$.
        \For{$k=0,\dots,E-2$}
            \State Collect trajectories $\mathcal{T}_k = \left\{\{\flow_t^{\pi_{k}}(\xi_i^k)\}_{t=0}^{T-1}\right\}_{i=1}^{m/E}.$ \label{line:collect_traj}
            \State $\hat{\pi}_k {\gets \texttt{cERM}\left( \mathcal{T}_k, \pi_k, c_k, 0\right)}$.
            \State $\pi_{k+1} \gets (1-\alpha) \pi_k + \alpha \hat{\pi}_k$. \label{line:update}
        \EndFor
        \State Collect trajectories $\mathcal{T}_{E-1} = \left\{\{\flow_t^{\pi_{E-1}}(\xi_i)\}_{t=0}^{T-1}\right\}_{i=1}^{m/E}.$ \label{line:collect_traj_em1}
        \State $c_{E-1} \gets \frac{(1-\alpha)^E}{\alpha} \frac{1}{m/E} \sum_{i=1}^{m/E} \ell_{\pi_{E-1}}(\xi_i^{E-1}; \pi_{E-1}, \pi_\star)$.
        \State $\hat\pi_{E-1} {\gets \texttt{cERM}}\left( \mathcal{T}_{E-1}, \pi_{E-1}, c_{E-1},  (1-\alpha)^E\right)$.
        \State $\pi_E \gets \frac{1}{1-(1-\alpha)^E} [(1-\alpha) \pi_{E-1} + \alpha \hat{\pi}_{E-1} - (1-\alpha)^E \pi_\star]$.
        \State \textbf{return} $\pi_E$.
    \end{algorithmic}
    \caption{Constrained Mixing Iterative Learning (CMILe)}
    \label{alg:csmile}
\end{algorithm}
\begin{algorithm}[htb]
    \centering
    \begin{algorithmic}[1]
        \Statex \textbf{Input:} trajectories $\left\{\{\flow_t^{\pi_{\mathrm{roll}}}(\xi_i)\}_{t=0}^{T-1}\right\}_{i=1}^{m}$, 
        policy $\pi_{\mathrm{roll}} \in \Pi$, constraint $c \geq 0$, weight $w \in [0, 1)$.
        \State \textbf{return} the solution to:
        \begin{subequations}
        \label{eq:csmile_opt}
        \begin{align}
                &\minimize_{\bar\pi\in\Pi} \:\: \textstyle\frac{1}{m} \sum_{i=1}^{m} \ell_{\pi_{\mathrm{roll}}}(\xi_i; \bar{\pi}, \pi_\star) \label{eq:opt:cost} \\
                &\mathrm{subject\text{ }to} \:\: \textstyle \frac{1}{m} \sum_{i=1}^{m}\ell_{\pi_{\mathrm{roll}}}(\xi_i; \bar{\pi}, \pi_{\mathrm{roll}}) \leq c, \label{eq:opt:trust_region} \\
                &\tfrac{1}{1-w} [(1-\alpha)\pi_{\mathrm{roll}} + \alpha \bar\pi - w \pi_\star ] \in \calS(a, b, \Psi). \label{eq:inc_gain_stability_constraint}
        \end{align} 
        \end{subequations}
    \end{algorithmic}
    \caption{Constrained Empirical Risk Minimization \texttt{cERM}\bigg{(}$\left\{\{\flow_t^{\pi_{\mathrm{roll}}}(\xi_i)\}_{t=0}^{T-1}\right\}_{i=1}^{m}$, $\pi_{\mathrm{roll}}$, $c$, $w$\bigg{)}}
    \label{alg:constrained_erm}
\end{algorithm}

With these definitions and assumptions in place, we introduce IGS-Constrained Mixing Iterative Learning (CMILe) in Algorithm~\ref{alg:csmile} and state our main
theoretical results. CMILe draws upon and integrates ideas from Stochastic Mixing Iterative Learning (SMILe)~\citep{ross2010efficient}, constrained policy optimization~\citep{schulman2015trust,luo2018algorithmic}, and the IGS tools developed in Section~\ref{sec:inc_gain_stability}.  As in SMILe and DAgger, CMILe proceeds in epochs, beginning with data generated by the expert policy, and iteratively shifts towards a learned policy via updates of the form $\pi_{k+1} = (1-\alpha)\pi_k + \alpha\hat\pi_k$, where $\pi_k$ is the current data-generating policy, $\hat\pi_k$ is the policy learned using the most recently generated data, and $\alpha\in(0,1]$ is a mixing parameter.   However, CMILe contains two key departures from traditional IL algorithms: (i) it constrains the learned policy at each epoch to remain appropriately close to the previous epoch's data-generating policy (constraint \eqref{eq:opt:trust_region}), and (ii) all data-generating policies $\{\pi_k\}$ are constrained to induce IGS closed-loop systems (constraint~\eqref{eq:inc_gain_stability_constraint}).  The latter constraint allows us to leverage the IGS machinery of Section~\ref{sec:inc_gain_stability} to analyze Algorithm~\ref{alg:csmile}.

In presenting our results, we specialize the policy class $\Pi$ to have the parametric form:
\begin{align}\label{eq:smooth_pi}
    \Pi = \{ \pi(x, \theta) \mid \theta \in \R^q, \, \norm{\theta}_2 \leq B_\theta \},
\end{align}
with $B_\theta \geq 1,$ and $\pi$ a fixed twice continuously differentiable map.  As an example, neural networks with $q$ weights and twice continuously differentiable activation functions are captured by the policy class \eqref{eq:smooth_pi}.
We note that our results do not actually require a parameteric representation: as long as a particular
policy class Rademacher complexity (defined in Appendix~\ref{sec:app:main_proof}) can be bounded, then our results apply.
In what follows, we define the following constants:
\begin{align*}
    L_{\partial^2 \pi} &:= 1 \vee \sup_{\norm{x} \leq \zeta B_0^{\alpha_0}, \norm{\theta} \leq B_\theta} \bignorm{\frac{\partial^2 \pi}{\partial \theta \partial x}}_{\ell^2(\R^q) \rightarrow M( \R^{d \times n})}, \\
    \bar{L} &:= \max\{B_\theta L_{\partial^2 \pi}, L_\Delta \} .
\end{align*}
Here, $M(\R^{d \times n})$ is the Banach space of $d \times n$ real-valued matrices
equipped with the operator norm.

\paragraph{IGS-Constrained Behavior Cloning.} We first analyze a single epoch version of Algorithm 1, which reduces to Behavior Cloning (BC) subject to interpolating the expert policy on the training data (constraint \eqref{eq:opt:trust_region}) and inducing an IGS closed-loop system (constraint \eqref{eq:inc_gain_stability_constraint}).

\begin{restatable}[IGS-BC]{mythm}{mainbc}
\label{thm:main_bc}
Suppose that Assumption~\ref{assumption:main}
and Assumption~\ref{assumption:stability} hold. Set $\alpha = E = 1$ in Algorithm~\ref{alg:csmile}. 
Suppose that $m$ satisfies:
\begin{align*}
    m \gtrsim (\zeta B_0^{\alpha_0} \bar{L})^2 \cdot  T^{2(1-1/a)} \cdot q.
\end{align*}
With probability at least $1-e^{-q}$ over the randomness of Algorithm~\ref{alg:csmile}, we have that:
\begin{align*}
    \E_{\xi\sim\calD} \ell_{\pi_1}(\xi;\pi_1,\pi_\star) \lesssim \gamma L_\Delta (\zeta B_0^{\alpha_0} \bar{L})^{1/a} \cdot  T^{1-1/a^2} \cdot  \left(\frac{q}{m}\right)^{\tfrac{1}{2a}}.
\end{align*}
\end{restatable}

Theorem~\ref{thm:main_bc} shows that the imitation loss for IGS-constrained BC decays as $T^{1-1/a^2} \cdot (\tfrac{q}{m})^{\frac{1}{2a}}$  We discuss implications on sample-complexity after analyzing the general setting.

\paragraph{IGS-CMILe.} Next we analyze Algorithm 1 as stated, and show that if the mixing parameter $\alpha$ and number of episodes $E$ are chosen appropriately with respect to the IGS parameters of the underlying expert system, sample-complexity guarantees similar to those in the IGS-constrained BC setting can be obtained.  As described above, the key to ensuring that guarantees can be bootstrapped across epochs is the combination of a trust-region constraint \eqref{eq:opt:trust_region} and IGS-stability constraints \eqref{eq:inc_gain_stability_constraint} on the intermediate data-generating policies.

\begin{restatable}[IGS-CMILe]{mythm}{mainshift}
\label{thm:main_shift}
Suppose that Assumption~\ref{assumption:main} and Assumption~\ref{assumption:stability}
hold, and that:
\begin{align}
    m \gtrsim (\zeta B_0^{\alpha_0} \bar{L})^2 \cdot  T^{2(1-1/a)} \cdot E^{2a+1} (q \vee \log{E})  \label{eq:main_shift_m_req}
\end{align}
Suppose further that for $k\in\{1, \dots, E-2\}$, we have:
\begin{align}
    c_k \lesssim \zeta B_0^{\alpha_0} \bar{L} \cdot T^{1-1/a} \cdot \sqrt{\frac{E (q \vee \log{E})}{m}}, \label{eq:main_shift_c_req}
\end{align}
that $E$ divides $m$, $E \geq \frac{1}{\alpha} \log\left(\frac{1}{\alpha}\right)$,
and $\alpha \leq \min\left\{ \frac{1}{2}, \frac{1}{L_\Delta \gamma T^{1-1/a}} \right\}$.
Then with probability at least $1-e^{-q}$ over the randomness of Algorithm~\ref{alg:csmile}, 
Algorithm~\ref{alg:csmile} is feasible for all epochs, and:
\begin{align*}
    \E_{\xi\sim\calD} \ell_{\pi_E}(\xi;\pi_E,\pi_\star) &\lesssim L_\Delta \gamma (\zeta B_0^{\alpha_0} \bar{L})^{1/a^2} \cdot T^{(1-1/a)(1+1/a^2)} \cdot \left( \frac{E^{2a+1} (q \vee \log{E})}{m} \right)^{\frac{1}{2a^2}}.
\end{align*}
\end{restatable}

Theorem~\ref{thm:main_shift} states that if the mixing parameter $\alpha$ and number of episodes $E$ are set according to the underlying IGS-stability parameters of the expert system then the imitation loss of the final policy $\pi_E$ scales as $T^{\left(1-\frac{1}{a}\right)\left(1+\frac{1}{a}+\frac{3}{2a^2}\right)} \cdot \left(\frac{q}{m}\right)^{\tfrac{1}{2a^2}}$.

\paragraph{Comparing IGS-BC and IGS-CMILe.}
By comparing the sample complexity bound for IGS-BS (Theorem~\ref{thm:main_bc}) to the bound for IGS-CMILe (Theorem~\ref{thm:main_shift}), we see that for fixed IGS parameters
and number of trajectories $m$, the imitation error
for IGS-BC is order-wise dominated by the IGS-CMILe imitation error.
That is, while our current analysis does show the benefits of expert
robustness via explicit dependence on the IGS parameters, it
does not show the relative
benefits of IGS-CMILe over IGS-BC, despite our experimental evidence
suggesting otherwise (cf.~Section~\ref{sec:experiments}).
We leave a theoretical analysis showing the benefit of IGS-CMILe
over IGS-BC to future work.

\paragraph{Sublinear rates.} From the above discussion, we can delineate classes of systems for which imitation loss sample-complexity bounds are sublinear in the task horizon $T$. Specifically, we bound the
number of trajectories $m$ needed to achieve $\e$-bounded imitation loss,
ignoring logarithmic factors and 
problem constants except the horizon length
$T$ and the IGS parameter $a$:
\begin{itemize}
    \item \textbf{IGS-BS} (Theorem~\ref{thm:main_bc}) requires $m \gtrsim \e^{-2a} \cdot T^{2a(1-1/a^2)}$ trajectories; this is sublinear in $T$ when $a \in [1, (1+\sqrt{17})/4) \approx [1, 1.281)$.
    \item \textbf{IGS-CMILe} (Theorem~\ref{thm:main_shift}) requires $m \gtrsim \e^{-2a^2} \cdot T^{2a^2\left(1-\frac{1}{a}\right)\left(1+\frac{1}{a}+\frac{3}{2a^2}\right)}$ trajectories; this is sublinear in $T$ when $a \in [1, (3/2)^{1/3}) \approx [1, 1.144)$.
\end{itemize}
Finally, when a system is contracting, 
we have $a=1$ by Proposition~\ref{prop:inc_gain_contraction},
and hence the number of required trajectories $m$ for
both IGS-BC and IGS-CMILe simplifies to $m \gtrsim \e^{-2}$.
 
\paragraph{Constraints.}
The requirement on the constraint slack $c_k$ in \eqref{eq:main_shift_c_req}
allows the constrained ERM problem (Algorithm~\ref{alg:constrained_erm})
non-zero slack in matching the behavior of the
previous policy (cf.~\eqref{eq:opt:trust_region}).
This is compatible with practical implementations of 
first order trust region policy optimization, 
where constraints are enforced via soft losses
instead of as hard constraints.

\subsection{Necessity of Stability Constraints}
\label{sec:results:necessity_of_stability}
The empirical risk minimization algorithm (Algorithm~\ref{alg:constrained_erm}) we consider in this work
requires an IGS constraint on the learned policy
(cf.~\eqref{eq:inc_gain_stability_constraint}).
Here, we show the necessity of imposing this stability constraint in order
to derive high probability sub-exponential in $T$ bounds on the imitation error.
Consider the linear time-invariant system:
\begin{align*}
    x_{t+1} = f(x_t, u_t) = A x_t + u_t, \quad A = \diag(2, 2, 0, \dots, 0).
\end{align*}
Let the expert policy be $\pi_\star(x) = -A x_t$.
Observe that $\flow_t^{\pi_\star}(\xi) = 0$
for all $t \in \N_+$ and $\xi \in \R^n$.
Hence, for any $m$
initial conditions $\xi_1, \dots, \xi_m$, 
the only non-zero state/action pairs provided by the expert are
$\{(\xi_i, y_i := \pi_\star(\xi_i))\}_{i=1}^{m}$.

Let $m_0$ be large enough so that $(1-1/m)^m \geq 1/(2e)$
for all $m \geq m_0$
(such an $m_0$ exists since $\lim_{m \rightarrow \infty} (1-1/m)^m = 1/e$), and fix any $m \geq m_0$.
Consider $\calD$ defined
as $\Pr_{\xi \sim \calD}(\xi = e_1) = 1-1/m$
and $\Pr_{\xi \sim \calD}(\xi = e_2) = 1/m$,
where $e_i \in \R^n$ is the $i$-th standard
basis vector.
Let $\calE_m := \bigcap_{i=1}^{m} \{ \xi_i = e_1 \}$, and observe that
$\Pr(\calE_m) = (1-1/m)^m \geq 1/(2e)$.
We will consider optimization over the compact, convex policy class 
$$\Pi = \{ x \mapsto K x \mid K \in \R^{n \times n}, \,\norm{K}_F \leq 2 \sqrt{2} \},$$ 
which contains $\pi_\star$.
Note that Assumptions \ref{assumption:main} and \ref{assumption:stability} hold for the closed-loop expert
dynamics and policy class.
The behavior cloning ERM problem (without stability constraints on $K$) is:
\begin{align*}
    \argmin_{\pi \in \Pi} \left[\frac{1}{m} \sum_{i=1}^{m} \norm{ \pi(\xi_i) - y_i }_2 \right] = \argmin_{\substack{K \in \R^{n \times n} \textrm{ s.t.} \\ \norm{K}_F \leq 2\sqrt{2}}} \left[\frac{1}{m} \sum_{i=1}^{m} \norm{K \xi_i - y_i}_2\right].
\end{align*}
It is easy to check that on $\calE_m$
(a constant probability event), $\hat{K} = - 2 e_1e_1^\T$
is an optimal solution of this ERM problem (that achieves zero training loss).
Let $\hat{\pi}(x) = \hat{K} x$.
Since $\flow_t^{\hat{\pi}}(e_2) = 2^t e_2$,
\begin{align*}
    \E_{\xi \sim \calD} \ell_{\hat{\pi}}(\xi; \hat{\pi}, \pi_\star) \geq \frac{1}{m} \sum_{t=1}^{T} \norm{\hat{\pi}(\flow_t^{\hat{\pi}}(e_2)) - \pi_\star(\flow_t^{\hat{\pi}}(e_2))}_2 
    = \frac{1}{m} \sum_{t=1}^{T} 2^t \norm{ A e_2 }_2 = \frac{4 (2^T - 1)}{m}.
\end{align*}
Therefore,
if one removes the stability constraint \eqref{eq:inc_gain_stability_constraint}, then 
sub-exponential in $T$ imitation error bounds are impossible without
more problem assumptions or other algorithmic modifications.

%% file: experiments.tex
\section{Experiments}
\label{sec:experiments}

In our experiments, we implement neural network training by combining the
\verb|haiku| NN library~\citep{haiku2020github}
with \verb|optax|~\citep{optax2020github}
in \verb|jax|~\citep{jax2018github}. 

\subsection{Practical Algorithm Implementation}
In order to implement Algorithm~\ref{alg:csmile}, a constrained ERM subproblem (Algorithm~\ref{alg:constrained_erm}) over the policy class must be solved.  Two elements make this subproblem practically challenging: (i) the trust-region constraint \eqref{eq:opt:trust_region}, and (ii) the IGS-stability constraint \eqref{eq:inc_gain_stability_constraint}. 

\paragraph{Trust-regions constraints.} 
We implement the trust-region constraint \eqref{eq:opt:trust_region} by
initializing the weights $\hat\theta_k$ 
parameterizing the policy $\hat\pi_k$ 
at those of the previous epoch's parameters $\hat\theta_{k-1}$, and using a small learning rate during training. Alternative viable approaches include imposing trust-region constraints on the parameters of the form $\|\hat\theta_k-\hat\theta_{k-1}\|_2\leq\kappa$, or explicitly enforcing the trust-region constrain \eqref{eq:opt:trust_region}.  These latter options would be implemented through either a suitable Lagrangian relaxation to soft-penalties in the objective, or by drawing on recent results in constrained empirical risk minimization~\citep{chamon2020probably}.

\paragraph{Enforcing IGS.} Enforcing the IGS constraint \eqref{eq:inc_gain_stability_constraint} via an incremental Lyapunov function (cf. Proposition~\ref{prop:lyap_characterization}) is more challenging, as it must be enforced for all $x$ within a desired region of attraction. If such a Lyapunov function is known for the expert, then it can be used to only enforce stability constraints on trajectory data, an approximation/heuristic that is common in the constrained policy optimization literature (see for example~\citep{schulman2015trust,luo2018algorithmic}).  However, if a Lyapunov certificate of IGS stability for the expert is not known, then options include (a) learning such a certificate for the expert from data, see for example~\citep{boffi2020learning,kenanian2019data}, or (b) jointly optimizing over an IGS certificate and learned policy.  Although this may be computationally challenging, alternating optimization schemes have been proposed and successfully applied in other contexts, see for example~\citep{singh2020learning,manek2020learning,chang2020neural,giesl2016approximation}.

Fortunately, we note that empirically, explicit stability constraints seem not to be required.
In the next subsection, we study the quantitative effects of enforcing the stability constraint \eqref{eq:inc_gain_stability_constraint} for a linear system, for which a stability certificate is available, and for which the level of stability of the expert system can be quantitatively tuned.  We observe that only when (a) the expert is nearly unstable, and (b) we are in a low-data regime,
that a small difference in performance between stability-constrained and unconstrained algorithms occurs, 
suggesting that optimal policies are naturally stabilizing.  
Therefore, we simply omit constraint \eqref{eq:inc_gain_stability_constraint} from our implementation and take care to ensure that sufficient data is provided to the IL algorithms to yield stabilizing policies.

\subsection{Tuneable IGS System} 

We consider the dynamical system in $\R^{10}$:
\begin{align}
    x_{t+1} = x_t - \eta x_t \frac{\abs{x_t}^p}{1 + \abs{x_t}^p} + \frac{\eta}{1 + \abs{x_t}^p}(h(x_t) + u_t), \quad \eta=0.3. \label{eq:toy_p_system}
\end{align}
All arithmetic operations in \eqref{eq:toy_p_system} are element-wise.
We set $h : \R^{10} \rightarrow \R^{10}$ to be a randomly initialized two layer MLP with zero biases, hidden width $32$, and $\tanh$ activations.
The expert policy is set to be $\pi_\star = -h$, so that the expert's closed-loop
dynamics are given by
$x_{t+1} = x_t - \eta x_t \frac{\abs{x_t}^p}{1 + \abs{x_t}^p}$.
From Proposition~\ref{prop:inc_gain_p}, the incremental stability
of the closed-loop system degrades as $p$ increases.

\begin{table}[ht]
\small
\centering 
\begin{tabular}{ccccccc} 
\hline\hline 
$p$ & \textbf{BC+IGS} & \textbf{BC} & \textbf{CMILe+IGS} & \textbf{CMILe} & \textbf{DAgger+IGS} & \textbf{DAgger} \\
\hline
$1$ & $0.149 \pm 0.020$ & $0.335 \pm 0.073$ & $0.167 \pm 0.013$ & $0.199 \pm 0.047$ & $0.195 \pm 0.036$ & $0.318 \pm 0.081$ \\
$2$ & $0.454 \pm 0.032$ & $0.782 \pm 0.158$ & $0.510 \pm 0.018$ & $0.692 \pm 0.026$ & $0.419 \pm 0.020$ & $0.624 \pm 0.127$ \\
$3$ & $0.829 \pm 0.131$ & $1.128 \pm 0.118$ & $0.852 \pm 0.057$ & $1.099 \pm 0.046$ & $0.654 \pm 0.020$ & $0.764 \pm 0.134$ \\
$4$ & $1.220 \pm 0.176$ & $1.737 \pm 0.126$ & $1.041 \pm 0.045$ & $1.412 \pm 0.052$ & $0.834 \pm 0.027$ & $0.924 \pm 0.107$ \\
$5$ & $1.899 \pm 0.160$ & $2.067 \pm 0.214$ & $1.236 \pm 0.035$ & $1.535 \pm 0.042$ & $0.992 \pm 0.018$ & $0.948 \pm 0.049$ \\
\hline 
\end{tabular}
\caption{Final $\norm{x_T^{\mathrm{expert}} - x_T^{\mathrm{IL}}}_2$ of imitation learning algorithms on \eqref{eq:toy_p_system}.}
\label{table:toy_p_delta_goal_err}
\end{table}

\begin{table}[ht]
\small
\centering 
\begin{tabular}{ccccccc} 
\hline\hline 
$p$ & \textbf{BC+IGS} & \textbf{BC} & \textbf{CMILe+IGS} & \textbf{CMILe} & \textbf{DAgger+IGS} & \textbf{DAgger} \\
\hline
$1$ & $13.800 \pm 1.359$ & $13.160 \pm 2.869$ & $17.592 \pm 0.658$ & $19.174 \pm 1.388$ & $3.186 \pm 0.535$ & $4.727 \pm 0.883$ \\
$2$ & $13.317 \pm 1.644$ & $16.523 \pm 3.874$ & $20.215 \pm 1.284$ & $23.968 \pm 1.692$ & $4.309 \pm 0.277$ & $5.443 \pm 1.416$ \\
$3$ & $19.559 \pm 4.886$ & $20.294 \pm 6.720$ & $24.235 \pm 2.221$ & $28.302 \pm 1.176$ & $5.179 \pm 0.172$ & $5.521 \pm 1.618$ \\
$4$ & $33.476 \pm 16.543$ & $56.552 \pm 16.436$ & $25.244 \pm 2.104$ & $33.181 \pm 2.769$ & $5.574 \pm 0.315$ & $5.958 \pm 0.857$ \\
$5$ & $85.239 \pm 16.620$ & $89.692 \pm 37.347$ & $28.639 \pm 2.050$ & $34.137 \pm 3.056$ & $6.325 \pm 0.268$ & $5.777 \pm 0.319$ \\
\hline 
\end{tabular}
\caption{Final average closed-loop imitation loss error of imitation learning algorithms
on \eqref{eq:toy_p_system}.}
\label{table:toy_p_imitation_err}
\end{table}

In this experiment, we vary $p \in \{1, \dots, 5\}$ to see the effect of $p$ on
the final task goal error and imitation loss.
We compare three different algorithms. \textbf{BC} is standard behavior cloning.
\textbf{CMILe} is Algorithm~\ref{alg:csmile}
with the practical modifications as described above.
\textbf{DAgger} is the imitation learning algorithm from \citet{ross2011reduction}.
For each algorithm, we also 
consider a modification (indicated by the \textbf{+IGS} label)
where policy imitation is augmented with a soft loss
encoding the IGS constraint \eqref{eq:inc_gain_stability_constraint}.
For all algorithms, we fix the number of trajectories $m$
from \eqref{eq:toy_p_system} to be $m=250$.
The horizon length is $T=100$.
The distribution $\calD$ over initial condition is set as $N(0, I)$.
We set the policy class $\Pi$ to be two layer MLPs with hidden width $64$
and $\tanh$ activations.
Each algorithm minimizes the imitation loss using $300$ epochs of Adam 
with learning rate $0.01$ and batch size $512$.
For all algorithms except \textbf{BC}, we use $E=25$ epochs 
with $\alpha = 0.15$ (in DAgger's notation, we set $\beta_k = 0.85^k$),
resulting in $10$ trajectories per epoch.

In Table~\ref{table:toy_p_delta_goal_err},
we track the difference in norm $\norm{x_T^{\mathrm{expert}}-x_T^{\mathrm{IL}}}_2$ between the expert's final state
($x_T^{\mathrm{expert}}$) and the IL algorithm's final state ($x_T^{\mathrm{IL}}$), both seeded
seeded from the same initial conditions.
In Table~\ref{table:toy_p_imitation_err}, we track the
final average closed-loop imitation loss error $\frac{1}{T}\E_{\xi\sim\calD} \ell_{\pi_E}(\xi;\pi_E,\pi_\star)$.
The entries in the tables are computed
by rolling out $500$ test trajectories and computing the median
quantity over the test trajectories. Each algorithm is repeated for $50$ trials, and 
the median quantity $\pm \max(\text{80th percentile} - \text{median}, \text{median} - \text{20th percentile})$ (over the $50$ trials) is shown.
In Table~\ref{table:toy_p_delta_goal_err}, we see that as $p$ decreases, the
goal deviation error decreases, showing that the task becomes fundamentally
easier. This trend is also reflected in all the imitation learning algorithms. 
Table~\ref{table:toy_p_imitation_err} provides insight into why
the goal deviation error decreases with $p$, and also shows
that our main theorems are indeed predictive:
as $p$ decreases,
the closed-loop average imitation loss $\frac{1}{T} \E_{\xi\sim\calD} \ell_{\pi_E}(\xi;\pi_E,\pi_\star)$ generally decreases for all algorithms.
Finally, comparing a baseline algorithm with its \textbf{+IGS} variant,
we see generally that for a fixed $p$, the 
IGS constrained variant has both lower final goal deviation error
and lower imitation loss, showing that 
stability constraints reduce sample-complexity by trimming the hypothesis space.

\subsection{Unitree Laikago}
We now study IL on
the Unitree Laikago robot, an 18-dof quadruped with 3-dof of actuation per leg.
We use PyBullet~\citep{coumans2021} for our simulations.
The goal of this experiment is to demonstrate, much like for the previous tuneable family of IGS systems, that increasing the stability
of the underlying closed-loop expert decreases the sample-complexity of imitation learning.
We do this qualitatively by studying a sideways walking task 
where the robot tracks a constant sideways linear velocity.
By increasing the desired linear velocity, the resulting expert closed-loop becomes more unstable.

Our expert controller is a model-based predictive controller
using a simplified center-of-mass dynamics as described in~\citet{dicarlo18mitcheetah}.
The stance and swing legs are controlled separately. The swing leg controller is based
on a proportional-derivative (PD) controller.
The stance leg controller solves for the desired contact forces to be applied at the foot
using a finite-horizon constrained linear-quadratic optimal
control problem; the linear model is computed from linearizing the center-of-mass dynamics.
The desired contact forces are then converted to hip motor torques using the body Jacobian. 
More details about the expert controller can be found in the appendix.

We restrict our imitation learning to the stance leg controller, as it is
significantly more complex than the swing leg controller. 
Furthermore, instead of randomizing over initial conditions, we inject randomization
into the environment by subjecting the Laikago to a sequence of random push forces throughout
the entire trajectory. We
compare the performance of \textbf{BC}, \textbf{CMILe}, and \textbf{CMILe+Agg}; the \textbf{CMILe+Agg} algorithm is identical to \textbf{CMILe}, except that at epoch $k$,
the data from previous epochs $j \in \{0, \dots, k-1\}$ is also used
in training. 
\textbf{DAgger} is omitted
for space reasons as its performance is comparable to \textbf{CMILe+Agg}.

We set the horizon length to $T=1000$, and 
featurized the robot state into a $14$-dimensional feature vector; the exact features are given in Appendix~\ref{sec:appendix:laikago:features}.
The output of the policy is a $12$-dimensional vector ($x,y,z$ contact forces for each of the $4$ legs).
We used a policy class of two layer MLPs of hidden width $64$ with ReLU activations.
For training, we ran $500$ epochs of Adam with a batch size of $512$ and step size of $0.001$. 
Furthermore, we tried to overcome the effect of overfitting in \textbf{BC} by using
the following heuristic: we used $5\%$ of the training data as a holdout set, and we stopped training
when either the holdout risk increased $h=50$ times or $500$ epochs were completed, whichever came first.
To assess the effect of the number of samples on imitation learning, we vary the number
of rollouts per epoch $S \in \{1, \dots, 5\}$. 
For \textbf{CMILe} and \textbf{CMILe+Agg},
we fix $\alpha=0.3$ and $E=12$.
We provide \textbf{BC} with $S \times E$ total trajectories.

\begin{figure}[htb]
\centering
\includegraphics[width=0.92\textwidth]{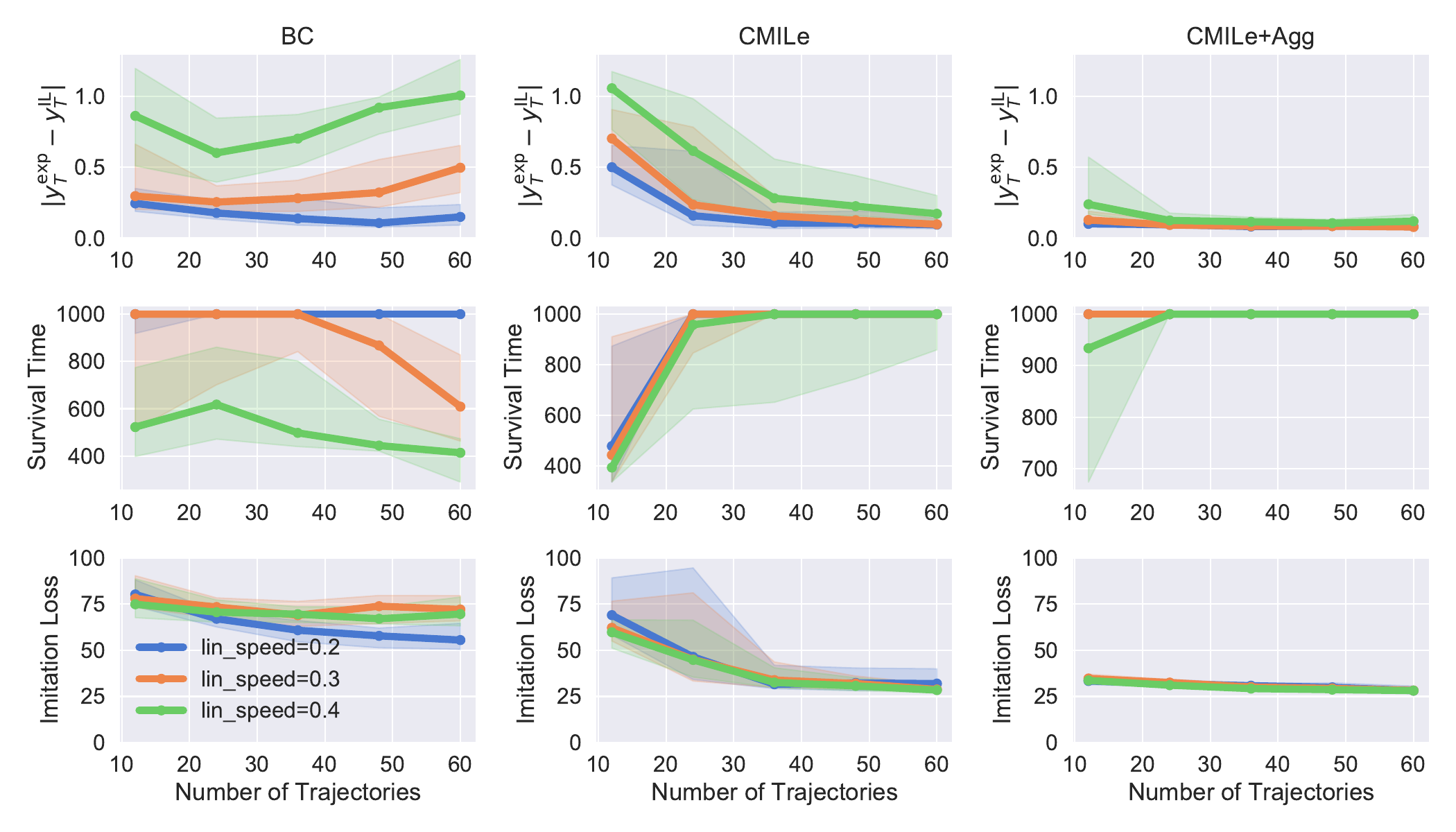}
\caption{Imitation learning on a sideways walking task.
The top row shows the deviation error $\abs{y_T^{\mathrm{exp}} - y_T^{\mathrm{IL}}}$
of the various algorithms,
the middle row shows the survival times,
and the bottom row plots the average closed-loop imitation
loss $\frac{1}{T} \E_{\xi\sim\calD} \ell_{\pi_E}(\xi;\pi_E,\pi_\star)$.}
\label{fig:laikago_walk_sideways}
\end{figure}

Figure~\ref{fig:laikago_walk_sideways} shows the result of our experiments.
In the top row, we plot
the deviation error $\abs{y_T^{\mathrm{exp}} - y_T^{\mathrm{IL}}}$
between the expert's final $y$ position ($y_T^{\mathrm{exp}}$)
and the imitation learning algorithm's final $y$ position ($y_T^{\mathrm{IL}}$). Both positions are measured in meters.
We observe that as the target linear speed (measured in $m/s$) decreases, 
the deviation between the expert and IL algorithms also decreases;
this qualitative trend is consistent with Theorem~\ref{thm:main_bc}
and Theorem~\ref{thm:main_shift}.
Note that the $y$ positions are computed by subjecting the Laikago
to the same sequence of random force pushes for both the expert and IL algorithm.
For the IL algorithm, if the rollout terminates early,
then the last $y$ position (before termination) is used in place
of $y_T^{\mathrm{IL}}$.
In the middle row, we plot the survival times for each of the algorithms, which
is the number of simulation steps (out of $1000$) that the robot successfully executes
before a termination criterion triggers which indicates the robot is about to fall.
We see that for all algorithms, by decreasing the sideways linear velocity,
the resulting learned policy is able to avoid falling more.
Note that for \textbf{CMILe+Agg}, the learned policy does not fall for
linear speeds $0.2$ and $0.3$.
In the bottom row, we plot the average closed-loop imitation loss
$\frac{1}{T} \E_{\xi\sim\calD} \ell_{\pi_E}(\xi;\pi_E,\pi_\star)$.
We see that for \textbf{BC}, the imitation loss shows improvement with increased samples for linear speed of $0.2$, but no improvements occur for the more difficult linear speeds of $0.3$ and $0.4$. This trend is less apparent for
\textbf{CMILe} and \textbf{CMILe+Agg}, but is most prominently seen in the deviation error $|y_T^{\mathrm{exp}}-y_T^{\mathrm{IL}}|$.

%% file: conclusion.tex
\section{Conclusions and Future Work}
\label{sec:conclusion}
We showed that IGS-constrained IL algorithms allow for a granular connection between the stability properties of an underlying expert system and the resulting sample-complexity of an IL task.  Our future work will focus on two complementary directions.  First, CMILe and DAgger significantly outperform BC in our experiments, but our bounds are not yet able to capture this: future work will look to close this gap.  Second, although our focus in this paper has been on imitation learning, we have developed a general framework for reasoning about learning over trajectories in continuous state and action spaces. We will look to apply our framework in other settings, such as safe exploration and model-based reinforcement learning.

%% file: appendix/experiments-stability.tex
\section{Stability Study}
\label{sec:appendix:stability}

\subsection{Constructing a Robust Lyapunov Function}
 In order to compute a robust Lyapunov function for use in the stability experiments of Section~\ref{sec:stability}, we use the following approach. Let $A_{\mathrm{lqr}}:= A + BK_{\mathrm{lqr}}$ and consider the following Lyapunov equation:
\begin{align}
A_{\mathrm{lqr}}^\T X A_{\mathrm{lqr}} - \gamma^2 X + \varepsilon I = 0, \label{eq:lqr_lyap}
\end{align}
for $\gamma\in(0,1)$ and $\varepsilon>0$.  Note that by rewriting this equation as \[
\left(\frac{A_{\mathrm{lqr}}}{\gamma}\right)^\T X \left(\frac{A_{\mathrm{lqr}}}{\gamma}\right) - X + \frac{\varepsilon}{\gamma^2}I = 0,
\] we see that this equation has a unique positive define solution so long as $A_{\mathrm{lqr}}/\gamma$ is stable, i.e., so long as $\gamma\geq \rho(A_{\mathrm{lqr}})$ \citep{zhou1996robust}.

Then note that we can rewrite the Lyapunov equation \eqref{eq:lqr_lyap} as
\[
A_{\mathrm{lqr}}^\T X A_{\mathrm{lqr}} - X + Q = 0,
\]
for $Q=(1-\gamma^2)X + \varepsilon I$.  To that end, we suggest solving the Lyapunov equation
\[
A_{\mathrm{lqr}}^\T X A_{\mathrm{lqr}} - X + Q = 0,
\]
with $Q = (1-\gamma^2)P_\star + \varepsilon I$, for a small $\varepsilon > 0$ and $P_\star$ the solution to the DARE, so as to obtain a a Lyapunov certificate with similar convergence properties to that of the expert, where $\varepsilon$ trades off between how small $\gamma$ can be and how robust the resulting Lyapunov function is to mismatches between the learned policy and the expert policy.  We note that the existence of solutions to this Lyapunov equation are guaranteed by continuity of the solution of the Lyapunov equation and that the solution to the DARE $P_\star$ is the maximizing solution among symmetric solutions \citep{dullerud2013course}.

\subsection{Experimental Results}

We study the effects of explicitly constraining the played policies $\{\pi_i\}_{i=1}^E$ to be IGS through the use of Lyapunov certificates.
The use of Lyapunov constraints to enforce incremental
stability was studied in Section G.1 of \citet{boffi2020regret},
with the main takeaway being that systems satisfying suitable exponential Lyapunov conditions, in particular those certifying exponential input-to-state-stability, are also
exponentially IGS (i.e., satisfy $a=b=1$) on a compact set of initial conditions and
bounded inputs.

In order to easily tune the underyling stability properties of the resulting expert system, we consider linear quadratic (LQ) control of a linear system $x_{t+1} = Ax_t + Bu_t$, where $A\in\R^{n\times n}$ and $B\in\R^{n\times d}$.
The LQ control problem can be expressed as
\begin{alignat}{2}
    &\minimize_{\{x_t\}, \{u_t\}}\ &&  \textstyle\sum_{t=0}^{\infty} x_t^\T Qx_t + u_t^\T Ru_t \label{eq:lqr-1} \\
    &\mathrm{subject\text{ }to} && x_{t+1} = Ax_t + Bu_t, \, x_0 = \xi. \nonumber
\end{alignat}
where $Q\in\R^{n \times n}$ and $R\in\R^{d \times d}$ are fixed cost matrices.  The optimal policy is {linear} in the state $x_t$, i.e., $u_t = K_{\mathrm{lqr}} x_t$, where $K_{\mathrm{lqr}}\in\R^{d \times n}$ can be computed by solving a discrete-time algebraic Riccati equation \citep{zhou1996robust}.  For our study, we set the task horizon $T=25$, the state $x_t\in\R^{10}$, the control input $u_t\in\R^4$, and we fix a randomly generated but unstable set of dynamics $(A,B)$.
Specifically, our realization satisfied
$\opnorm{A} = 5.893$, $\opnorm{B} = 4.964$,
and the open loop system was unstable with spectral radius $\rho(A) = 3.638$. 
We also set $R=I_4$, and $Q=\nu I_{10}$, and vary $\nu$ across three orders of magnitude: $\nu\in\{0.0001,0.001,0.01\}$.  In the limit of $\nu\to 0$, the optimal LQ controller is the minimum energy stabilizing controller, whereas for larger $\nu$, the optimal LQ controller balances between state-deviations and control effort.  We synthesize the optimal state-feedback LQR controller for the dynamics $(A,B)$ and prescribed cost matrices $(Q,R)$ by solving the Discrete Algebraic Riccati Equation (DARE) using \verb|scipy.linalg.solve_discrete_are|. 
The resulting closed loop LQR norms of the resulting systems for $\nu=0.01$, $0.001$, and $0.0001$ were $10.909$, $10.784$, and $10.739$ respectively.

We drew initial conditions according to the distribution $N(0, 4)$.  We used a policy parameterized by a two-hidden-layer feed-forward neural network with ReLU activations.  Each hidden layer in this network had a width of 64 neurons.  For both BC and CMILe without stability constraints, we train the policy for 500 epochs; for CMILe with stability constraints, we trained policies for 1000 epochs.  All neural networks were optimized with the Adam optimizer with a learning rate of $0.01$.

\begin{figure}[t]
    \centering
    \includegraphics[width=0.8\columnwidth]{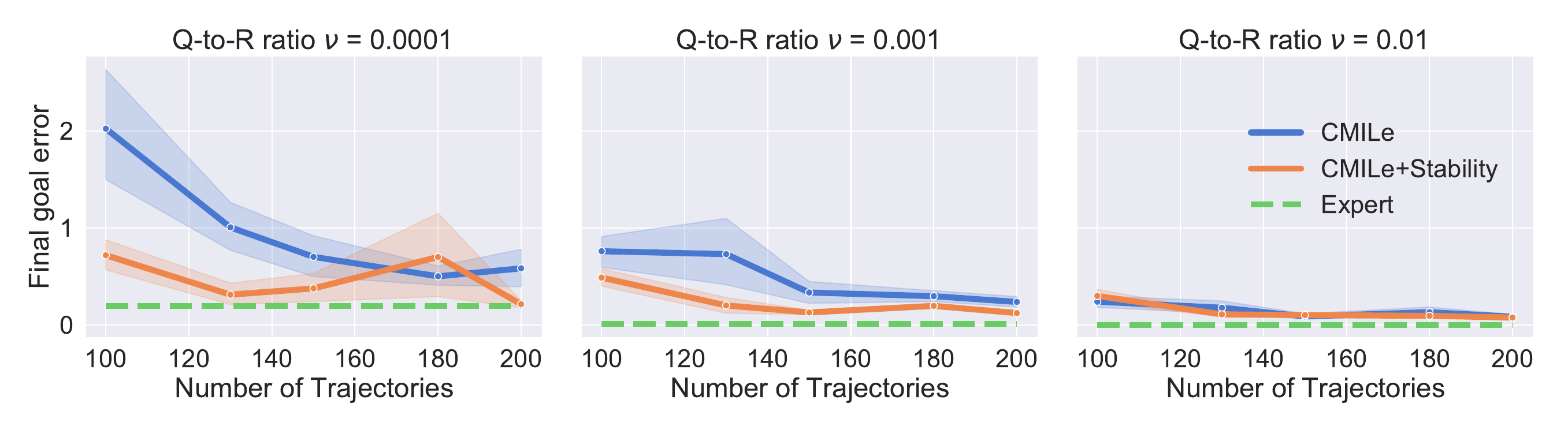}
    \caption{For fixed system matrices $(A,B)$ and cost matrices $R=I_q$ and $Q = \nu I_q$ for $\nu\in\{0.0001, 0.001, 0.01\}$, we show the median goal error obtained by CMILe with and without stability constraints over 100 i.i.d.\ test trajectories. }
    \label{fig:linear-system-errors}
\end{figure}

In Figure~\ref{fig:linear-system-errors}, we plot the median goal error $\norm{x_T}_2$ achieved by policies learned via CMILe for different values of $\nu$, both with and without Lyapunov stability constraints, over $100$ test rollouts.
The error bars represent the $20$th/$80$th percentiles of the median across ten independent trials.   We construct a valid robust Lyapunov certificate from the solution to the DARE (see Appendix \ref{sec:appendix:stability} for details), and use it to enforce 
the IGS constraint \eqref{eq:inc_gain_stability_constraint}
on the resulting closed loop dynamics $f_{cl}^{\pi_k}(x)= Ax + B\pi_k(x)$.  Two important trends can be observed in Figure~\ref{fig:linear-system-errors}. First, the smaller the weight $\nu$, the more dramatic the effect of the stability constraint, i.e., when the underlying expert is itself fragile (closed-loop spectral radius $\rho(A+BK_{\mathrm{lqr}})\approx 1$), stability constraints have a measurable effect. Second, the effect of stability constraints is more dramatic in low-data regimes, and by restricting $\pi_k\in \Pi_\Psi$, we reduce over-fitting.  We also evaluated the performance of standard behavior cloning (BC), but even with 200 training trajectories and $\nu=0.01$, the median final goal error was $504$.

%% file: appendix/experiments-laikago.tex
\section{Laikago Experimental Details}
\label{sec:appendix:laikago}

\subsection{More Details on Expert Controller}
The expert controller contains multiple components: the swing controller, the stance controller, and the gait generator. The gait generator uses the clock source to generate a desired gait pattern, where a pair of diagonal legs are synchronized and are out of phase with the other pair. Throughout our experiments, we fixed the gait to be a trotting gait. The swing controller generates the aerial trajectories of the feet when they lift up and controls the landing positions based on the desired moving speed. The stance leg controller is based on model predictive control (MPC) using centroidal dynamics~\citep{dicarlo18mitcheetah}. 
Recall that in our experiments, we only perform imitation learning for the stance leg controller.

\begin{figure}[ht]
\centering
\includegraphics[width=0.5\textwidth]{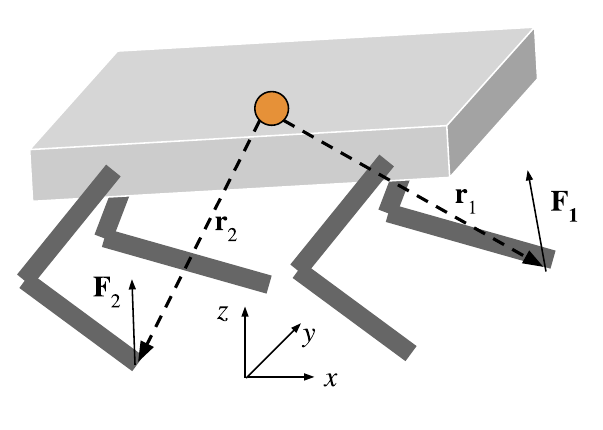}
\caption{The centroidal dynamics model used to formulate the model predictive control problem.}
\label{fig:mpc}
\end{figure}

We now describe the centroidal dynamics model. We treat the whole robot as a single rigid body, and assume that the inertia contribution from the leg movements is negligible (cf. Figure~\ref{fig:mpc}). With these assumptions, the system dynamics can be simply written using the Newton-Euler equations:
\begin{align*}
    m\mathbf{\ddot{x}} &= \sum_{i=1}^4 \mathbf{F}_i - \mathbf{g}, \\
    \frac{d}{dt}(\mathbf{I}\mathcal{\omega}) &= \sum_{i=1}^4 \mathbf{r}_i \times \mathbf{F}_i,
\end{align*}

where $\mathbf{x}=(x, y, z, \Phi,\Theta, \Psi)$ denotes the center of mass (CoM) translation and rotation, $\mathbf{F}_i = (f_{x}, f_{y}, f_{z})_i$ is the contact force applied on the $i$-th foot (set to zero if the $i$-th foot is not in contact with the ground), and $\mathbf{r}_i$ is the displacement from the CoM to the contact point. We used the same $Z-Y-X$ Euler angle conventions in~\citep{dicarlo18mitcheetah} to represent the CoM rotation. Since the robot operates in a regime where its base is close to flat, singularity from the Euler angle representation is not an concern. 

The MPC module solves an optimization problem over a finite horizon $H$ to track a desired pose and velocity $\mathbf{q} = (\mathbf{x}, \mathbf{\dot{x}})$. The system dynamics can be linearized around the desired state and discretized:
\begin{align*}
    \mathbf{q}_{t+1} = A(\mathbf{q}^d) \mathbf{q}_t + B(\mathbf{q}^d)\mathbf{u}_t,
\end{align*}
where $\mathbf{u}_t = (\mathbf{F}_{1, t}, \mathbf{F}_{2, t}, \mathbf{F}_{3, t}, \mathbf{F}_{4, t})$, is the concatenated force vectors from all feet. We then formally write the optimization target:
\begin{align*}
    \min_{\mathbf{u}_t} &\quad \sum_{t=1}^H [(\mathbf{q}_t - \mathbf{q}^d_t)^T\mathbf{Q}(\mathbf{q}_t - \mathbf{q}^d_t) + \mathbf{u}_t^T\mathbf{R}\mathbf{u}_t], \\
    \textrm{s.t.} &\quad \mathbf{q}_{t+1} = A \mathbf{q}_t + B\mathbf{u}_t, \\
                  &\quad \mathbf{F}_{i,t} = 0 \:\:\textrm{if $i$-th foot is not in contact}, \\
                  &\quad 0 \leq f_{z,t} \leq f_{max}, \quad\textrm{contact normal force for each foot}, \\
                  &\quad -\mu f_{z, t} \leq f_{x, t} \leq \mu f_{z, t}, \\
                  &\quad -\mu f_{z, t} \leq f_{y, t} \leq \mu f_{z, t}, \quad\textrm{friction cone},
\end{align*}
where we used a diagonal $\mathbf{Q}$ and $\mathbf{R}$ matrix with the weights detailed in~\citep{dicarlo18mitcheetah}. At runtime, we apply the feet contact forces from the first step by converting them to motor torques using the Jacobian matrix.

\subsection{Featurization}
\label{sec:appendix:laikago:features}
The inputs to the MPC algorithm is a 28 dimension vector $(\mathbf{q}, \mathbf{q}^d, \mathbf{c}, \mathbf{r})$, where $\mathbf{c}$ is a binary vector indicating feet contact states, and $\mathbf{r} = (\mathbf{r}_1, \mathbf{r}_2, \mathbf{r}_3, \mathbf{r}_4)$ represent the relative displacement from the CoM to each feet. This representation contains redundant information, since one can infer the body height $z$ from the contact state, and local feet displacements when the quadruped is walking on flat ground. Also, since in our experiments the desired pose and speed of the robot are fixed (moving along $y$ direction at constant speed without body rotation), they are not passed as inputs to the imitation policy. Furthermore, the current body linear velocities of the robot are omitted from the inputs as well, since they are not directly measurable on a legged robots without motion capture systems or state estimators. As a result, the inputs to the imitation policy are condensed to a 14 dimensional vector $(\Phi,\Theta, \mathbf{r\cdot c})$, i.e., the roll, pitch angle of the CoM, and the contact state masked feet positions.

%% file: appendix/incremental-gain-stability.tex
\section{Incremental Gain Stability Proofs}
\label{sec:app:inc_gain_stability}

\subsection{Preliminaries}

We first prove a simple proposition which we use repeatedly.
\begin{myprop}
\label{prop:min_max_many_to_two}
Let $a_1, \dots, a_k \in \R$.
Then for any $x \in \R$, we have:
\begin{align*}
    \min\{ \abs{x}^{a_1}, \dots, \abs{x}^{a_k} \} &= \min\{ \abs{x}^{\min\{a_1, \dots, a_k\}}, \abs{x}^{\max\{a_1, \dots, a_k\}} \}, \\
    \max\{ \abs{x}^{a_1}, \dots, \abs{x}^{a_k} \} &= \max\{ \abs{x}^{\min\{a_1, \dots, a_k\}}, \abs{x}^{\max\{a_1, \dots, a_k\}} \}.
\end{align*}
\end{myprop}
\begin{proof}
We only prove the result for $\min$ since the proof for $\max$ is nearly identical.
First, suppose that $\abs{x} \leq 1$. Then, the map $a \mapsto \abs{x}^a$ is decreasing on $\R$, and hence:
\begin{align*}
    \min\{ \abs{x}^{a_1}, \dots, \abs{x}^{a_k} \} = \abs{x}^{\max\{a_1, \dots, a_k\}} = \min\{ \abs{x}^{\min\{a_1, \dots, a_k\}}, \abs{x}^{\max\{a_1, \dots, a_k\}} \}.
\end{align*}
Now, suppose that $\abs{x} > 1$. The map $a \mapsto \abs{x}^a$ is increasing on $\R$, and hence:
\begin{align*}
    \min\{ \abs{x}^{a_1}, \dots, \abs{x}^{a_k} \} = \abs{x}^{\min\{a_1, \dots, a_k\}} = \min\{ \abs{x}^{\min\{a_1, \dots, a_k\}}, \abs{x}^{\max\{a_1, \dots, a_k\}} \}.
\end{align*}
The claim now follows.
\end{proof}

We now derive some basic consequences of the definition of incremental gain stability.
The following helper proposition will be useful for what follows.
\begin{myprop}
\label{prop:signal_holder}
Fix any initial conditions $\xi_1, \xi_2 \in \R^n$,
and any signal $\{u_t\}_{t \geq 0}$. Let
\begin{align*}
    \Delta_t := \flow_t(\xi_1, \{u_t\}_{t \geq 0}) - \flow_t(\xi_2, \{0\}_{t \geq 0}).
\end{align*}
For any $a \in [1, \infty)$ and 
integers $T_1, T_2$ satisfying $0 \leq T_1 \leq T_2$, we have:
\begin{align*}
    \sum_{t=T_1}^{T_2} \norm{\Delta_t}_X \leq \left( \sum_{t=T_1}^{T_2} \min\{\norm{\Delta_t}_X, \norm{\Delta_t}_X^{a}\} \right)^{1/a} (T_2-T_1+1)^{1 - 1/a} + \sum_{t=T_1}^{T_2} \min\{\norm{\Delta_t}_X, \norm{\Delta_t}_X^{a}\}.
\end{align*}
\end{myprop}
\begin{proof}
Let the index set $I \subseteq \{ T_1, \dots, T_2 \}$ be defined as:
\begin{align*}
    I := \{ t \in \{T_1, \dots, T_2\} \mid \norm{\Delta_t}_X \leq 1 \}.
\end{align*}
By H{\"{o}}lder's inequality, since $a \in [1, \infty)$,
\begin{align*}
    \sum_{t=T_1}^{T_2} \norm{\Delta_t}_X &= \sum_{t \in I} \norm{\Delta_t}_X + \sum_{t \in I^c} \norm{\Delta_t}_X \\
    &\leq \left( \sum_{t\in I} \norm{\Delta_t}_X^{a} \right)^{1/a} \abs{I}^{1 - 1/a} + \sum_{t \in I^c} \norm{\Delta_t}_X \\
    &= \left( \sum_{t\in I} \min\{\norm{\Delta_t}_X, \norm{\Delta_t}_X^{a}\} \right)^{1/a} \abs{I}^{1 - 1/a} + \sum_{t \in I^c} \min\{\norm{\Delta_t}_X, \norm{\Delta_t}_X^{a}\} \\
    &\leq \left( \sum_{t=T_1}^{T_2} \min\{\norm{\Delta_t}_X, \norm{\Delta_t}_X^{a}\} \right)^{1/a} (T_2-T_1+1)^{1 - 1/a} + \sum_{t=T_1}^{T_2} \min\{\norm{\Delta_t}_X, \norm{\Delta_t}_X^{a}\}.
\end{align*}
\end{proof}

Next, we compare the autonomous trajectories between two different
initial conditions (both trajectories are not driven by any input).
\begin{myprop}
\label{prop:inc_gain_stab_compare_ics}
Suppose that $f$ is $(a,b,\Psi)$-incrementally-gain-stable.
Fix a pair of initial conditions $\xi_1, \xi_2 \in X$
and define for $t \in \N$:
\begin{align*}
    \Delta_t := \flow_t(\xi_1, \{0\}_{t \geq 0}) - \flow_t(\xi_2, \{0\}_{t \geq 0}),
\end{align*}
We have for all $t \in \N$:
\begin{align}
    \norm{\Delta_t}_X \leq \zeta\max\left\{ \norm{\xi_1-\xi_2}_X^{\alpha_0}, \norm{\xi_1 - \xi_2}_X^{\alpha_0/a} \right\}. \label{eq:delta_state}
\end{align}
Furthermore, for any horizon $T \in \N_+$,
\begin{align}
    \sum_{t=0}^{T-1} \norm{\Delta_t}_X \leq 2 \zeta T^{1 - 1/a} \max\left\{ \norm{\xi_1-\xi_2}_X^{\alpha_0}, \norm{\xi_1 - \xi_2}_X^{\alpha_0/a} \right\}. \label{eq:sum_delta_states}
\end{align}
\end{myprop}
\begin{proof}

We first show \eqref{eq:delta_state}. Fix a $t \in \N$.
First, suppose that $\norm{\Delta_t}_X > 1$.
Then by $(a,b,\Psi)$-incremental-gain-stability,
\begin{align*}
    \norm{\Delta_t}_X = \min\{ \norm{\Delta_t}_X, \norm{\Delta_t}_X^{a} \} \leq \sum_{k=0}^{t} \min\{ \norm{\Delta_k}_X, \norm{\Delta_k}_X^{a} \} \leq \zeta \norm{\xi_1 - \xi_2}_X^{\alpha_0}.
\end{align*}
Now, suppose that $\norm{\Delta_t}_X \leq 1$.
By a similar argument:
\begin{align*}
    \norm{\Delta_t}_X \leq \left[ \zeta \norm{\xi_1 - \xi_2}_X^{\alpha_0} \right]^{1/a}.
\end{align*}
Combining these inequalities yields the desired inequality \eqref{eq:delta_state}:
\begin{align*}
    \norm{\Delta_t}_X &\leq \max\left\{ \zeta \norm{\xi_1 - \xi_2}_X^{\alpha_0}, \left[ \zeta \norm{\xi_1 - \xi_2}_X^{\alpha_0} \right]^{1/a} \right\} \\
    &\leq \zeta \max\left\{ \norm{\xi_1-\xi_2}_X^{\alpha_0}, \norm{\xi_1 - \xi_2}_X^{\alpha_0/a} \right\}.
\end{align*}

Now we turn to \eqref{eq:sum_delta_states}.
By Proposition~\ref{prop:signal_holder}
and $(a,b,\Psi)$-incremental-gain-stability, we have:
\begin{align*}
    \sum_{t=0}^{T-1} \norm{\Delta_t}_X \leq \left[ \zeta \norm{\xi_1-\xi_2}_X^{\alpha_0} \right]^{1/a} T^{1 - 1/a}  + \zeta \norm{\xi_1-\xi_2}_X^{\alpha_0}.
\end{align*}
Therefore:
\begin{align*}
    \sum_{t=0}^{T-1} \norm{\Delta_t}_X \leq 2 \zeta T^{1 - 1/a} \max\left\{ \norm{\xi_1-\xi_2}_X^{\alpha_0}, \norm{\xi_1 - \xi_2}_X^{\alpha_0/a} \right\}.
\end{align*}
\end{proof}

The next result compares two trajectories starting
from the same initial condition, but 
one being driven by an input sequence $\{u_t\}$
whereas the other is autonomous.
\begin{myprop}
\label{prop:inc_gain_stab_compare_inputs}
Suppose that $f$ is $(a,b,\Psi)$-incrementally-gain-stable.
Then, for all $T \in \N_+$, all initial conditions $\xi \in X$ and
all input signals $\{u_t\}_{t \geq 0} \subseteq U$,
letting
\begin{align*}
    \Delta_t := \flow_t(\xi, \{u_t\}_{t \geq 0}) - \flow_t(\xi, \{0\}_{t \geq 0}),
\end{align*}
we have:
\begin{align*}
    \sum_{t=1}^{T} \norm{\Delta_t}_X &\leq 4 \gamma T^{1-1/a} \max\left\{
    \left(\sum_{t=0}^{T-1} \norm{u_t}_U\right)^{1/a}, \left(\sum_{t=0}^{T-1} \norm{u_t}_U\right)^{b}
    \right\}.
\end{align*}
\end{myprop}
\begin{proof}
By Proposition~\ref{prop:signal_holder}, the fact that $\Delta_0 = 0$,
and $(a,b,\Psi)$-incremental-gain-stability,
\begin{align*}
    \sum_{t=1}^{T} \norm{\Delta_t}_X 
    \leq \left( \gamma \sum_{t=0}^{T-1} \max\{ \norm{u_t}_U, \norm{u_t}_U^b \}  \right)^{1/a} T^{1 - 1/a} + \gamma \sum_{t=0}^{T-1} \max\{ \norm{u_t}_U, \norm{u_t}_U^{b} \}.
\end{align*}
From this, we conclude,
\begin{align*}
    \sum_{t=1}^{T} \norm{\Delta_t}_X \leq \gamma T^{1-1/a} \left[ \left(\sum_{t=0}^{T-1} \max\{ \norm{u_t}_U, \norm{u_t}_U^{b} \}  \right)^{1/a} +\left(\sum_{t=0}^{T-1} \max\{ \norm{u_t}_U, \norm{u_t}_U^{b} \}  \right) \right].
\end{align*}
Next, we observe that:
\begin{align*}
    \sum_{t=0}^{T-1} \max\{ \norm{u_t}_U, \norm{u_t}_U^{b} \}  \leq \sum_{t=0}^{T-1} \norm{u_t}_U + \sum_{t=0}^{T-1} \norm{u_t}_U^{b} \leq 2 \max\left\{\sum_{t=0}^{T-1} \norm{u_t}_U, \sum_{t=0}^{T-1} \norm{u_t}_U^{b}  \right\}.
\end{align*}
Therefore,
\begin{align*}
    &~~\left(\sum_{t=0}^{T-1} \max\{ \norm{u_t}_U, \norm{u_t}_U^{b} \}  \right)^{1/a} +\left(\sum_{t=0}^{T-1} \max\{ \norm{u_t}_U, \norm{u_t}_U^{b} \}  \right) \\
    &\leq 2^{1/a}\max\left\{ \left(\sum_{t=0}^{T-1} \norm{u_t}_U\right)^{1/a}, \left(\sum_{t=0}^{T-1} \norm{u_t}_U^{b}\right)^{1/a} \right\} + 2 \max\left\{ \sum_{t=0}^{T-1} \norm{u_t}_U, \sum_{t=0}^{T-1} \norm{u_t}_U^{b} \right\} \\
    &\leq 2\max\left\{ \left(\sum_{t=0}^{T-1} \norm{u_t}_U\right)^{1/a}, \left(\sum_{t=0}^{T-1} \norm{u_t}_U\right)^{b/a} \right\} + 2 \max\left\{ \sum_{t=0}^{T-1} \norm{u_t}_U, \left(\sum_{t=0}^{T-1} \norm{u_t}_U\right)^b \right\} \\
    &\leq 4 \max\left\{ \sum_{t=0}^{T-1}\norm{u_t}_U, \left(\sum_{t=0}^{T-1}\norm{u_t}_U\right)^{1/a}, \left(\sum_{t=0}^{T-1}\norm{u_t}_U\right)^{b}, \left(\sum_{t=0}^{T-1}\norm{u_t}_U\right)^{b/a}    \right\} \\
    &= 4 \max\left\{
    \left(\sum_{t=0}^{T-1} \norm{u_t}_U\right)^{1/a}, \left(\sum_{t=0}^{T-1} \norm{u_t}_U\right)^{b}
    \right\}.
\end{align*}
Above, the last equality follows from Proposition~\ref{prop:min_max_many_to_two}.
The claimed inequality now follows.
\end{proof}

\subsection{Proof of Proposition~\ref{prop:lyap_characterization}}

Let $\xi_1, \xi_2$ and $\{u_t\}_{t \geq 0}$ be arbitrary. Fix a $T \in \N_+$.
Define two dynamics, for $t = 0, \dots, T-1$:
\begin{alignat*}{2}
    x_{t+1} &= f(x_t, u_t), \qquad && x_0 = \xi_1, \\
    y_{t+1} &= f(y_t, 0), \qquad && y_0 = \xi_2.
\end{alignat*}
Now define $V_t := V(x_t, y_t)$.
Then for $t \in \{0, \dots, T-1\}$, by the assumed inequality \eqref{eq:inc_gain_stable_lyap_cond},
\begin{align*}
    V_{t+1} &= V(x_{t+1}, y_{t+1}) = V(f(x_t, u_t), f(y_t, 0)) \\
    &\leq V(x_t, y_t) - \mathfrak{a} \min\{\norm{x_t - y_t}_X, \norm{x_t - y_t}_X^{a}\} + \mathfrak{b} \max\{\norm{u_t}_U, \norm{u_t}_U^{b}\} \\
    &= V_t - \mathfrak{a} \min\{\norm{x_t - y_t}_X, \norm{x_t - y_t}_X^{a}\} + \mathfrak{b} \max\{\norm{u_t}_U, \norm{u_t}_U^{b}\}.
\end{align*}
Therefore, we have:
\begin{align*}
    V_T + \mathfrak{a} \sum_{t=0}^{T-1} \min\{ \norm{x_t - y_t}_X, \norm{x_t - y_t}_X^{a}\}  \leq V_0 + \mathfrak{b} \sum_{t=0}^{T-1} \max\{\norm{u_t}_U, \norm{u_t}_U^{b}\}.
\end{align*}
By \eqref{eq:inc_gain_stable_lyap_func}, this implies:
\begin{align*}
    \underline{\alpha}\norm{x_T - y_T}_X^{\alpha_0} + \mathfrak{a} \sum_{t=0}^{T-1} \min\{ \norm{x_t - y_t}_X, \norm{x_t - y_t}_X^{a}\}  \leq \overline{\alpha} \norm{\xi_1 - \xi_2}_X^{\alpha_0} + \mathfrak{b} \sum_{t=0}^{T-1} \max\{\norm{u_t}_U, \norm{u_t}_U^{b}\}.
\end{align*}
Next, by Proposition~\ref{prop:min_max_many_to_two}, since $\alpha_0 \in [1, a]$:
\begin{align*}
    &\underline{\alpha}\norm{x_T - y_T}_X^{\alpha_0} + \mathfrak{a} \sum_{t=0}^{T-1} \min\{ \norm{x_t - y_t}_X, \norm{x_t - y_t}_X^{a}\} \geq (\underline{\alpha} \wedge \mathfrak{a}) \sum_{t=0}^{T} \min\{ \norm{x_t-y_t}_X, \norm{x_t-y_t}_X^{a} \}.
\end{align*}
Therefore,
\begin{align*}
    \sum_{t=0}^{T} \min\{ \norm{x_t-y_t}_X, \norm{x_t-y_t}_X^{a} \} \leq \frac{\overline{\alpha}}{\underline{\alpha} \wedge \mathfrak{a}} \norm{\xi_1 - \xi_2}_X^{\alpha_0} + \frac{\mathfrak{b}}{\underline{\alpha} \wedge \mathfrak{a}} \sum_{t=0}^{T-1} \max\{ \norm{u_t}_U, \norm{u_t}_U^{b} \},
\end{align*}
which is the desired inequality.

%% file: appendix/examples.tex
\section{Examples of Incremental Gain Stability Proofs}

\subsection{Contraction}
\label{sec:app:examples:contraction}

Recall the following definition of autonomously contracting
in Proposition~\ref{prop:inc_gain_contraction}, which we duplicate below
for convenience.
\begin{mydef}
Consider the dynamics $x_{t+1} = f(x_t, u_t)$.
We say that $f$ is \emph{autonomously contracting} if
there exists a positive definite metric $M(x)$ and a scalar $\gamma \in (0, 1)$ such that:
\begin{align*}
    \frac{\partial f}{\partial x}(x, 0)^\T M(f(x, 0)) \frac{\partial f}{\partial x}(x, 0) \preceq \gamma M(x) \quad \forall x \in \R^n.
\end{align*}
\end{mydef}
For what follows, let $d_M$ denote the geodesic distance under the metric $M$:
\begin{align*}
    d_M(x, y) := \inf_{\gamma \in \Gamma(x, y)} \sqrt{ \int_0^1 \frac{\rmd \gamma}{\rmd s}(s)^\T M(\gamma(s)) \frac{\rmd \gamma}{\rmd s}(s)  \,\rmd s}. 
\end{align*}
Here, $\Gamma(x, y)$ is the set of smooth curves $\gamma : [0, 1] \rightarrow \R^n$ satisfying $\gamma(0) = x$
and $\gamma(1) = y$.
The next result shows that distances contract in the metric $d_M$
under an application of the dynamics $f$:
\begin{myprop}[cf. Lemma 1 of \citet{pham08discrete}]
\label{prop:metric_contraction}
Suppose that $f$ is autonomously contracting. Then for all $x, y \in \R^n$ we have:
\begin{align*}
    d^2_M(f(x, 0), f(y, 0)) \leq \gamma d^2_M(x, y) .
\end{align*}
\end{myprop}
The next result shows that the Euclidean norm lower and upper bounds the geodesic 
distance under $M$ as long as $M$ is uniformly bounded.
\begin{myprop}[cf. Proposition D.2 of \citet{boffi2020regret}]
\label{prop:metric_uniform_l2_bounds}
Suppose that $\underline{\mu} I \preceq M(x) \preceq \overline{\mu} I$
for all $x \in \R^n$.
Then for all $x, y \in \R^n$ we have:
\begin{align*}
    \sqrt{\underline{\mu}} \norm{x-y}_2 \leq d_M(x, y) \leq \sqrt{\overline{\mu}}\norm{x-y}_2.
\end{align*}
\end{myprop}
We now restate and prove Proposition~\ref{prop:inc_gain_contraction}.
\incgaincontraction*
\begin{proof}
Fix initial conditions $\xi_1, \xi_2$ and an input sequence $\{u_t\}_{t \geq 0}$.
Consider two systems:
\begin{alignat*}{2}
    x_{t+1} &= f(x_t, u_t), \qquad && x_0 = \xi_1, \\
    y_{t+1} &= f(y_t, 0), \qquad && y_0 = \xi_2.
\end{alignat*}
Now fix a $t \in \N$. We have:
\begin{align*}
    d_M(x_{t+1}, y_{t+1}) &= d_M(f(x_t, u_t), f(y_t, 0)) \\
    &\stackrel{(a)}{\leq} d_M(f(x_t, u_t), f(x_t, 0)) + d_M(f(x_t, 0), f(y_t, 0)) \\
    &\stackrel{(b)}{\leq} \sqrt{\gamma} d_M(x_t, y_t) + \sqrt{\overline{\mu}} \norm{f(x_t, u_t) - f(x_t, 0)}_2 \\
    &\stackrel{(c)}{\leq} \sqrt{\gamma} d_M(x_t, y_t) + L_u \sqrt{\overline{\mu}} \norm{u_t}_2.
\end{align*}
Above, (a) is triangle inequality, (b) follows
from Proposition~\ref{prop:metric_contraction} and Proposition~\ref{prop:metric_uniform_l2_bounds},
and (c) follows from the Lipschitz assumption.
Now unroll this recursion, to yield for all $t \geq 0$
\begin{align*}
    d_M(x_t, y_t) \leq \gamma^{t/2} d_M(x_0, y_0) + L_u \sqrt{\overline{\mu}} \sum_{k=0}^{t-1} \gamma^{(t-1-k)/2} \norm{u_k}_2. 
\end{align*}
Using the upper and lower bounds on $d_M$ from Proposition~\ref{prop:metric_uniform_l2_bounds}, we obtain for all $t \in \N$:
\begin{align*}
    \sqrt{\underline{\mu}} \norm{x_t - y_t}_2 \leq \gamma^{t/2} \sqrt{\overline{\mu}} \norm{x_0 - y_0}_2 + L_u \sqrt{\overline{\mu}} \sum_{k=0}^{t-1} \gamma^{(t-1-k)/2} \norm{u_k}_2.
\end{align*}
Now dividing both sides by $\sqrt{\underline{\mu}}$
and summing the left hand side,
\begin{align*}
    \sum_{t=0}^{T} \norm{x_t - y_t}_2 \leq \sqrt{\frac{\overline{\mu}}{\underline{\mu}}} \frac{1}{1-\sqrt{\gamma}} \norm{x_0 - y_0}_2 + L_u \sqrt{\frac{\overline{\mu}}{\underline{\mu}}} \frac{1}{1-\sqrt{\gamma}} \sum_{t=0}^{T-1} \norm{u_t}_2.
\end{align*}
\end{proof}

\subsection{Scalar Example with $p \in (0, \infty)$}
\label{sec:app:examples:scalar}

Recall we are interested in the family of systems:
\begin{align*}
    x_{t+1} = x_t - \eta x_t \frac{\abs{x_t}^p}{1+\abs{x_t}^p} + \eta u_t, \quad p \in (0, \infty).
\end{align*}
We define for convenience the function $h : \R \rightarrow \R$ as:
\begin{align*}
    h(x) := x \frac{\abs{x}^p}{1 + \abs{x}^p}, 
\end{align*}
and observe that $h(-x) = -h(x)$ for all $x \in \R$.
Our first proposition gives a lower bound which we utilize in the sequel.
\begin{myprop}
\label{prop:deriv_sign_inequality}
For every $x, y \in \R$,
\begin{align*}
    \sgn(x-y)\left[ x \frac{\abs{x}^p}{1 + \abs{x}^p} - y\frac{\abs{y}^p}{1 + \abs{y}^p} \right] \geq \frac{1}{2^{2+p}}\min\{ \abs{x-y}, \abs{x-y}^{1 + p} \}.
\end{align*}
\end{myprop}
\begin{proof}
Define 
\begin{align*}
    z(x, y) := \sgn(x-y) (h(x) - h(y)).
\end{align*}
It is straightforward to see that the following properties of $z$ hold
for all $x, y \in \R$:
\begin{enumerate}
    \item $z(x, y) = z(y, x)$.
    \item $z(-x, -y) = z(x, y)$.
\end{enumerate}
We want to show that for $x, y \in \R$,
\begin{align}
    z(x, y) \geq \frac{1}{2^{2+p}} \min\{ \abs{x-y}, \abs{x-y}^{1+p} \}. \label{eq:g_xy_lower_bound} 
\end{align}
Observe that
\begin{align*}
    z(0, y) = \frac{\abs{y}^{p+1}}{1 + \abs{y}^p} \geq \frac{1}{2} \min\{ \abs{y}, \abs{y}^{1+p} \},
\end{align*}
so \eqref{eq:g_xy_lower_bound} holds for $x=0$, $y \in \R$.
Next, by symmetry, $z(x, 0) = z(0, x)$,
so \eqref{eq:g_xy_lower_bound} also holds for $x \in \R$, $y=0$.
Furthermore, \eqref{eq:g_xy_lower_bound} holds for $(x,y)=(0,0)$ trivially.
Finally, we can assume that $x \leq y$ since $z(x, y) = z(y, x)$.
Therefore, for remainder of the proof, we may assume that $x \neq 0$, $y \neq 0$, and $x < y$.

\paragraph{Case 1: $0 < x < y$.}
Since $z \mapsto z / (1+z)$ is monotonically increasing on $\R_{\geq 0}$,
\begin{align*}
    z(x, y) = \frac{y^{1+p}}{1+y^{p}} - \frac{x^{1+p}}{1+x^{p}} 
    = y \frac{y^{p}}{1+y^{p}} - x \frac{x^{p}}{1+x^{p}} 
    \geq (y-x) \frac{y^{p}}{1+y^{p}}. 
\end{align*}
If $y > 1$, then we observe that $z(x, y) \geq (y-x)/2 = \frac{1}{2} \abs{x-y}$.
Now we assume $y \leq 1$.
Then, because $z \mapsto z^p$ is monotonically increasing on $\R_{\geq 0}$,
\begin{align*}
    (y-x) \frac{y^{p}}{1+y^{p}} \geq \frac{1}{2} (y-x) y^p \geq \frac{1}{2} (y-x) (y-x)^p = \frac{1}{2} (y-x)^{1+p} = \frac{1}{2} \abs{x-y}^{1+p}.
\end{align*}
Thus, \eqref{eq:g_xy_lower_bound} holds when $0 < x < y$.

\paragraph{Case 2: $x < 0 < y$.}
Let us assume wlog that $y \geq \abs{x}$, otherwise we can 
swap $x, y$ by considering $z(-y, -x) = z(x, y)$ instead.
Now, we have
\begin{align*}
    z(x, y) &= \frac{y^{1+p}}{1+y^{p}} + \frac{\abs{x}^{1+p}}{1 + \abs{x}^{p}} \geq \frac{y^{1+p}}{1 + y^p}.
\end{align*}
If $y > 1$, then we lower bound:
\begin{align*}
    z(x, y) \geq y/2 \geq (y + \abs{x})/4 = \frac{1}{4} \abs{x-y}.
\end{align*}
On the other hand, if $y \leq 1$, then
\begin{align*}
    z(x, y) \geq y^{1+p}/2 = \max\{y^{1+p},\abs{x}^{1+p}\}/2 \stackrel{(a)}{\geq} \frac{1}{2^{2+p}} ( y + \abs{x})^{1+p} = \frac{1}{2^{2+p}} \abs{x-y}^{1+p}.
\end{align*}
Here, (a) holds because 
for non-negative $a, b \in \R$, 
we have 
\begin{align*}
    (a+b)^{1+p} \leq (2 \max\{a, b\})^{1+p} = 2^{1+p} \max\{a^p, b^p\}.
\end{align*}
Therefore, \eqref{eq:g_xy_lower_bound} holds when $x < 0 < y$.

\paragraph{Case 3: $x < y < 0$.}
In this case, we have $z(x, y) = z(-y, -x)$.
By swapping $x, y$, we reduce to Case 1 where we know that
\eqref{eq:g_xy_lower_bound} already holds.
\end{proof}

Next, we show that the sign of $x-y$ is preserved
under a perturbation $x-y - \eta(h(x) - h(y))$,
as long as $\eta \geq 0$ is small enough.
\begin{myprop}
\label{prop:sign_preservation}
Let $\eta$ satisfy $0 \leq \eta < \frac{4}{5+p}$. We have that:
\begin{align*}
    \sgn\left( (x-y) - \eta\left[ x \frac{\abs{x}^p}{1 + \abs{x}^p} - y\frac{\abs{y}^p}{1 + \abs{y}^p} \right] \right) = \sgn(x-y).
\end{align*}
\end{myprop}
\begin{proof}
Define 
\begin{align*}
z(x, y) := (x-y) - \eta(h(x) - h(y)).
\end{align*}
We want to show that for all $x, y \in \R$,
\begin{align}
    \sgn(z(x, y)) = \sgn(x-y). \label{eq:goal_sign_perservation}
\end{align}
It is straightforward to see that for all $x, y \in \R$,
\begin{enumerate}
    \item $z(x, y) = -z(y, x)$,
    \item $z(-x, -y) = -z(x, y)$.
\end{enumerate}
Therefore, we have for all $x, y \in \R$:
\begin{align*}
    \sgn(z(x,y)) = \sgn(x - y) \Longrightarrow \{\sgn(z(y, x)) = \sgn(y - x)\} \bigwedge \{ \sgn(z(-x, -y)) = \sgn((-x) - (-y)) \}.
\end{align*}
Observe that \eqref{eq:goal_sign_perservation} holds
trivially when $x=y$.
Furthermore,
when $y > 0$:
\begin{align*}
    z(0, y) = -y + \eta \frac{y^{1+p}}{1+y^{p}}\leq -y + \eta y = -(1-\eta)y < 0.
\end{align*}
Therefore, \eqref{eq:goal_sign_perservation}
holds when $x = 0$ and $y > 0$,
which implies it also holds when
$x = 0$ and $y < 0$ (and hence it holds when
$x=0$ and $y \in \R$).
But this also implies that \eqref{eq:goal_sign_perservation} holds
when $x \in \R$ and $y = 0$.
Hence, as we did in Proposition~\ref{prop:deriv_sign_inequality},
we will assume that $x \neq 0$, $y \neq 0$, and $x < y$.

\paragraph{Case 1: $0 < x < y$.}

In this case we have:
\begin{align*}
    z(x, y) = (x-y) - \eta \left[ \frac{x^{1+p}}{1+x^{p}} - \frac{y^{1+p}}{1+y^{p}} \right].
\end{align*}
The derivative of $h(x)$ when $x > 0$ is:
\begin{align*}
    h'(x) = \frac{x^p(x^p + p + 1)}{(x^p+1)^2}.
\end{align*}
Clearly $h'(x) \geq 0$ when $x > 0$.
Furthermore, one can check 
that $\sup_{x \geq 0} \frac{x}{(1+x)^2} = 1/4$.
Therefore:
\begin{align*}
    h'(x) = \left(\frac{x^p}{1+x^p}\right)^2 + (p+1) \frac{x^p}{(1+x^p)^2} \leq 1 + (p+1)/4.
\end{align*}
Therefore:
\begin{align*}
    (x-y) - \eta \left[ \frac{x^{1+p}}{1+x^{p}} - \frac{y^{1+p}}{1+y^{p}} \right] &= (x-y) + \eta (h(y) - h(x)) \\
    &= (x-y) + \eta (y-x) \int_0^1 h'( t y + (1-t) x) \, \rmd t \\
    &\leq (x-y) + \eta (1 +(p+1)/4) (y-x) \\
    &= (1 - \eta(1 +(p+1)/4))(x-y).
\end{align*}
By assumption, we have that $1-\eta(1+(p+1)/4) > 0$ and hence
$z(x, y) < 0$.
This shows that $\sgn(z(x,y)) = \sgn(x-y)$,
and hence \eqref{eq:goal_sign_perservation} holds in this case.

\paragraph{Case 2: $x < 0 < y$.}
Here we have:
\begin{align*}
    z(x, y) &= -(\abs{x} + y) + \eta \left[\frac{\abs{x}^{1+p}}{1+\abs{x}^{p}} + \frac{{y}^{1+p}}{1+{y}^{p}}\right] \\
    &\leq -(\abs{x}+y) + \eta( \abs{x} + y) = -(1-\eta)(\abs{x} + y) < 0.
\end{align*}
Hence, we have $\sgn(z(x,y)) = \sgn(x-y)$, and 
hence \eqref{eq:goal_sign_perservation} holds in this case.

\paragraph{Case 3: $x < y < 0$.}
By a reduction to Case 1 when $0 < x < y$:
\begin{align*}
    \sgn(z(x, y)) = -\sgn(z(y, x)) = -\sgn(y-x) = \sgn(x-y).
\end{align*}
\end{proof}

\begin{myprop}
\label{prop:inc_gain_p_lyap}
Let $\eta$ satisfy $0 \leq \eta < \frac{4}{5+p}$. Put $V(x, y) = \abs{x-y}$
and $f(x, u) := x - \eta \frac{x \abs{x}^p}{1 + \abs{x}^p} + \eta u$.
Then we have for all $x, y, u \in \R$:
\begin{align*}
    V(f(x, u), f(y, 0)) - V(x, y) \leq -\frac{\eta}{2^{2+p}} \min\{\abs{x-y}, \abs{x-y}^{1+p}\} + \eta \abs{u}.
\end{align*}
\end{myprop}
\begin{proof}
Because $\sgn(x)$ is an element of $\partial \abs{x}$, by convexity of $\abs{\cdot}$,
\begin{align*}
   \abs{a} - \abs{b} \leq \sgn(a) (a-b).
\end{align*}
Therefore:
\begin{align*}
    &~~~~V(f(x, u),f(y)) - V(x,y) \\
    &= \bigabs{ (x-y) - \eta\left[ x \frac{\abs{x}^p}{1 + \abs{x}^p} - y\frac{\abs{y}^p}{1 + \abs{y}^p} \right] + \eta u} - \abs{x-y} \\
    &\leq \bigabs{ (x-y) - \eta\left[ x \frac{\abs{x}^p}{1 + \abs{x}^p} - y\frac{\abs{y}^p}{1 + \abs{y}^p} \right]} - \abs{x-y} + \eta \abs{u} \\
    &\leq \sgn\left((x-y) - \eta\left[ x \frac{\abs{x}^p}{1 + \abs{x}^p} - y\frac{\abs{y}^p}{1 + \abs{y}^p} \right] \right)\left\{  - \eta\left[ x \frac{\abs{x}^p}{1 + \abs{x}^p} - y\frac{\abs{y}^p}{1 + \abs{y}^p} \right] \right\} + \eta \abs{u} \\
    &\stackrel{(a)}{=} -\eta \sgn(x-y)\left[ x \frac{\abs{x}^p}{1 + \abs{x}^p} - y\frac{\abs{y}^p}{1 + \abs{y}^p} \right] + \eta \abs{u} \\
    &\stackrel{(b)}{\leq} -\frac{\eta}{2^{2+p}} \min\{\abs{x-y},\abs{x-y}^{1+p}\} + \eta \abs{u}.
\end{align*}
Above, (a) is Proposition~\ref{prop:sign_preservation}
and (b) is Proposition~\ref{prop:deriv_sign_inequality}.
\end{proof}
Note that Proposition~\ref{prop:inc_gain_p}
is an immediate consequence of 
Proposition~\ref{prop:inc_gain_p_lyap} with Proposition~\ref{prop:lyap_characterization}.

%% file: appendix/main-proof.tex
\section{Proof of Theorem~\ref{thm:main_bc} and Theorem~\ref{thm:main_shift}}
\label{sec:app:main_proof}

In this section we provide a theoretical analysis
of Algorithm~\ref{alg:csmile}.
Recall that $\Pi$ is our policy class and
$\calS(a, b, \Psi)$ is the set of all policies $\pi$ such that $\fcl^\pi$ is $(a, b, \Psi)$-incrementally-gain-stable.
Let us define $\Pi(a, b, \Psi) := \Pi \cap \calS(a, b, \Psi)$.
Finally, for any $\pi_d,\pi_1,\pi_2\in\Pi$, 
recall the definition of $\ell_{\pi_d}$:
\begin{align*}
    \ell_{\pi_d}(\xi; \pi_1, \pi_2) = \sum_{t=0}^{T-1} \norm{\Delta_{\pi_1,\pi_2}(\flow_t^{\pi_d}(\xi))}_2.
\end{align*}

\subsection{Uniform Convergence Toolbox}

Our main tool will be the following uniform convergence result.
\begin{myprop}
\label{prop:generalization_bounds}
Define $B_\ell$ to be the constant:
\begin{align}
    B_\ell := \sup_{\pi_d \in \Pi(a, 1, \Psi)} \sup_{\pi_1, \pi_2 \in \Pi} \esssup_{\xi \sim \calD} \ell_{\pi_d}(\xi; \pi_1, \pi_2). \label{eq:almost_sure_bound}
\end{align}
Next, define the following Rademacher complexity for the policy
class $\Pi$:
\begin{align}
    \calR_m(\Pi) &:= \sup_{\pi_d \in \Pi(a, 1, \Psi)} \sup_{\pi_g \in \Pi} \E_{\{\xi_i\}} \E_{\{\varepsilon_i\}}\left[ \sup_{\pi \in \Pi} \frac{1}{m} \sum_{i=1}^{m} \varepsilon_i \ell_{\pi_d}(\xi_i; \pi, \pi_g) \right]. \label{eq:rademacher-basic}
\end{align}
Now fix a data generating policy $\pi_d \in \Pi(a, 1, \Psi)$ and
goal policy $\pi_g \in \Pi$.
Furthermore, let $\xi_1, \dots, \xi_m$ be drawn i.i.d.\ from $\calD$.
With probability at least $1-\delta$ (over $\xi_1, \dots, \xi_m$), we have:
\begin{align}
    &\sup_{\pi \in \Pi} \bigabs{\E_{\xi \sim \calD} \ell_{\pi_d}(\xi; \pi, \pi_g) - \frac{1}{m} \sum_{i=1}^{m} \ell_{\pi_d}(\xi_i; \pi, \pi_g)} \leq 2 \calR_m(\Pi) + B_{\ell} \sqrt{\frac{\log(2/\delta)}{m}}, \label{eq:uniform_convergence}
\end{align}
\end{myprop}
\begin{proof}
This follows from standard uniform convergence results, see e.g.,~\citet{wainwright_book}.
\end{proof}

In order to use Proposition~\ref{prop:generalization_bounds}, we need to 
have upper bounds on the constants $B_\ell$ and $\calR_m(\Pi)$.
We first give an upper bound on $B_\ell$.
\begin{myprop}
\label{prop:almost_sure_bound}
Under Assumption~\ref{assumption:main}
and Assumption~\ref{assumption:stability}, we have that:
\label{prop:loss_bound}
\[
B_\ell = \sup_{\pi_d\in\Pi(a, 1, \Psi)} \sup_{\pi_1,\pi_2\in\Pi} \esssup_{\xi \sim \calD} \ell_{\pi_d}(\xi;\pi_1,\pi_2) \leq 2 \zeta B_0^{\alpha_0} L_\Delta T^{1-1/a}.
\]
\end{myprop}
\begin{proof}
Let $\pi_d \in \Pi(a, 1, \Psi)$ and $\pi_1,\pi_2\in\Pi$. Since $\Delta_{\pi_1,\pi_2}(0)=0$ and $\Delta_{\pi_1,\pi_2}$ is $L_\Delta$-Lipschitz:
\begin{align*}
\ell_{\pi_d}(\xi;\pi_1,\pi_2) &= \sum_{t=0}^{T-1} \norm{\Delta_{\pi_1,\pi_2}(\flow_t^{\pi_d}(\xi))}_2 \leq L_\Delta\sum_{t=0}^{T-1} \norm{\flow_t^{\pi_d}(\xi)}_2 \leq 2 \zeta B_0^{\alpha_0} L_\Delta T^{1-1/a}.
\end{align*}
Above, the last inequality follows from Proposition~\ref{prop:inc_gain_stab_compare_ics}.
\end{proof}

We now give a bound on 
the Rademacher complexity $\calR_m(\Pi)$.
\begin{myprop}\label{prop:rademacher_bound}
Let $\Pi = \{ \pi(x, \theta) \mid \theta \in \R^q, \, \norm{\theta}_2 \leq B_\theta \}$
for a fixed twice continuously differentiable map $\pi$. 
Define the constant $L_{\partial^2 \pi}$ to be:
\begin{align*}
    L_{\partial^2 \pi} = 1 \vee \sup_{\norm{x}_2 \leq \zeta B_0^{\alpha_0}, \norm{\theta}_2 \leq B_\theta} \bignorm{\frac{\partial^2 \pi}{\partial \theta \partial x}}_{\ell^2(\R^q) \rightarrow M( \R^{d \times n})} . 
\end{align*}
Here, $M(\R^{d \times n})$ is the Banach space of $d \times n$ real-valued matrices
equipped with the operator norm.
Under Assumption~\ref{assumption:main}
and Assumption~\ref{assumption:stability},
we have that:
\begin{align}
    \calR_m(\Pi) \leq  65 \zeta B_0^{\alpha_0}  B_\theta L_{\partial^2 \pi} T^{1-1/a} \sqrt{\frac{q}{m}} . \label{eq:rad_bound_basic}
\end{align}
\end{myprop}
\begin{proof}
Fix an $x$ and $\theta_1, \theta_2$. Since $\pi(0, \theta) = 0$ for all $\theta$, by repeated
applications of Taylor's theorem:
\begin{align*}
    \pi(x, \theta_1) - \pi(x, \theta_2) &= \left( \int_0^1 \left[\frac{\partial \pi}{\partial x}(s_1 x, \theta_1) - \frac{\partial \pi}{\partial x}(s_1 x, \theta_2)\right] \,\rmd s_1 \right) x \\
    &= \left( \int_0^1 \left[ \int_0^1 \frac{\partial^2 \pi}{\partial\theta \partial x}(s_1 x, s_2 \theta_1 + (1-s_2) \theta_2) (\theta_1 - \theta_2) \,\rmd s_2 \right] \, \rmd s_1 \right) x .
\end{align*}
Now supposing $\norm{x}_2 \leq \zeta B_0^{\alpha_0}$ and $\norm{\theta_i}_2 \leq B_\theta$ for $i \in \{1,2\}$, then:
\begin{align}
    \norm{\pi(x; \theta_1) - \pi(x; \theta_2)}_2 \leq L_{\partial^2 \pi} \norm{x}_2 \norm{\theta_1 - \theta_2}_2 . \label{eq:lipschitz_policy}
\end{align}
Next, we have that:
\begin{align*}
    \abs{\ell_{\pi_d}(\xi;\pi_1,\pi_g)-\ell_{\pi_d}(\xi;\pi_2,\pi_g)} 
    &\leq \sum_{t=0}^{T-1} \abs{ \norm{\Delta_{\pi_1,\pi_g}(\flow_t^{\pi_d}(\xi))}_2 - \norm{ \Delta_{\pi_2,\pi_g}(\flow_t^{\pi_d}(\xi)) }_2 }\\
    &\leq \sum_{t=0}^{T-1} \norm{\Delta_{\pi_1,\pi_2}(\flow_t^{\pi_d}(\xi))}_2 \\
    &\overset{(a)}{\leq} L_{\partial^2\pi}\left(\sum_{t=0}^{T-1}\|\flow_t^{\pi_d}(\xi)\|_2 \right)\|\theta_1-\theta_2\|_2 \\
    &\overset{(b)}{\leq} 2 \zeta B_0^{\alpha_0} L_{\partial^2 \pi} T^{1-1/a} \|\theta_1-\theta_2\|_2.
\end{align*}
Here,
(a) holds by Equation~\eqref{eq:lipschitz_policy}
and the fact that $\norm{\flow_t^{\pi_d}(\xi)}_2 \leq \zeta B_0^{\alpha_0}$
from Proposition~\ref{prop:inc_gain_stab_compare_ics}, 
and (b) by using Proposition~\ref{prop:inc_gain_stab_compare_ics} to bound $\sum_{t=0}^{T-1}\|\flow_t^{\pi_d}(\xi)\|_2$.
Now for a fixed $\xi_1, \dots, \xi_n$, $\pi_d$, and $\pi_g$, define the
empirical $\norm{\cdot}_{\mathbb{P}_m}$
metric over $\Pi$ as:
\begin{align*}
    \norm{\pi_1 - \pi_2}_{\mathbb{P}_m}^2 := \frac{1}{m} \sum_{i=1}^{m} (\ell_{\pi_d}(\xi_i; \pi_1, \pi_g) - \ell_{\pi_d}(\xi_i; \pi_2, \pi_g))^2 , \quad \pi_1, \pi_2 \in \Pi.
\end{align*}
The calculation above shows that for all $\pi_1, \pi_2 \in \Pi$,
\begin{align*}
    \norm{\pi_1 - \pi_2}_{\mathbb{P}_m} \leq 2 \zeta B_0^{\alpha_0} L_{\partial^2 \pi} T^{1-1/a} \norm{\theta_1 - \theta_2}_2.
\end{align*}
Thus, for every $\e > 0$, 
letting $L_\Pi := 2 \zeta B_0^{\alpha_0} B_\theta L_{\partial^2 \pi} T^{1-1/a}$,
we have the following upper bound on the covering number:
$$
  N(\e; \Pi, \norm{\cdot}_{\mathbb{P}_m}) \leq N\left(\frac{\e}{L_\Pi}; \mathbb{B}_2(1), \norm{\cdot}_2\right) \leq \left( 1 + \frac{2L_\Pi}{\e}\right)^q.
$$
Therefore by Dudley's entropy integral~(cf. \citet{wainwright_book}):
\begin{align*}
    \calR_m(\Pi) &\leq 24\sup_{\pi_d \in \Pi_\Psi}\sup_{\pi_g \in \Pi} \E_{\{\xi_i\}} \frac{1}{\sqrt{m}} \int_0^\infty \sqrt{\log{N(\varepsilon; \Pi, \norm{\cdot}_{\mathbb{P}_m})}} \,\rmd\varepsilon \\
    &\leq \frac{24}{\sqrt{m}} \int_0^{L_\Pi} \sqrt{q \log\left(1 + \frac{2L_\Pi}{\e} \right)} \,\rmd\varepsilon \\
    &\leq 32.5 L_\Pi \sqrt{\frac{q}{m}}.
\end{align*}
The last inequality above follows from the numerical estimate:
\begin{align*}
    \int_0^1 \sqrt{\log\left(1+\frac{2}{\e}\right)} \,\rmd \varepsilon \leq 1.353.
\end{align*}
This yields \eqref{eq:rad_bound_basic}.
\end{proof}

\subsubsection{Rademacher Complexity Under Convex Hulls}

Let $\conv(\Pi)$ denote the convex hull of the policy class $\Pi$,
and let 
$$\conv(\Pi)(a, b, \Psi) := \conv(\Pi) \cap \calS(a, b, \Psi).$$
The following auxiliary proposition shows that the Rademacher complexity
of the policy class $\conv(\Pi)$ can be analyzed nearly identically to the 
Rademacher complexity of the original class $\Pi$.
\begin{myprop}
\label{prop:rademacher_convex_hull}
Suppose that the policy class $\Pi$ is uniformly bounded,
i.e., 
$$\sup_{\norm{x}_2 \leq \zeta B_0^{\alpha_0}} \sup_{\pi \in \Pi} \norm{\pi(x)}_2 < \infty.$$
We have that:
\begin{align*}
    \calR_m(\conv(\Pi)) = \sup_{\pi_d \in \conv(\Pi)(a, 1, \Psi)} \sup_{\pi_g \in \conv(\Pi)} \E_{\{\xi_i\}} \E_{\{\varepsilon_i\}}\left[ \sup_{\pi \in \Pi} \frac{1}{m} \sum_{i=1}^{m} \varepsilon_i \ell_{\pi_d}(\xi_i; \pi, \pi_g) \right].
\end{align*}
\end{myprop}
\begin{proof}
The following argument is based on the proof of \citet[Theorem 12]{bartlett2002rademacher}.
Fix $\pi_d$, $\pi_g$, $\{\xi_i\}$ and $\{\e_i\}$,
and let $\mathbb{S}^{d-1}$ denote the unit sphere
in $\R^d$.
By the variational representation of the Euclidean norm, we have that:
\begin{align*}
    &~~~\,\sup_{\pi \in \conv(\Pi)} \frac{1}{m}\sum_{i=1}^{m}\e_i \ell_{\pi_d}(\xi_i; \pi, \pi_g) \\
    &= \sup_{\pi \in \conv(\Pi)} \frac{1}{m}\sum_{i=1}^{m} \sum_{t=0}^{T-1}\e_i \norm{\Delta_{\pi,\pi_g}(\flow_t^{\pi_d}(\xi_i))}_2 \\
    &= \sup_{\{v_{i,t}\} \subset \mathbb{S}^{d-1}} \sup_{\pi \in \conv(\Pi)} \frac{1}{m}\sum_{i=1}^{m}\sum_{t=0}^{T-1} \e_i \ip{v_{i,t}}{\pi(\flow_t^{\pi_d}(\xi_i)) - \pi_g(\flow_t^{\pi_d}(\xi_i))} \\
    &= \sup_{\{v_{i,t}\} \subset \mathbb{S}^{d-1}} \left(\left[ \sup_{\pi \in \conv(\Pi)} \frac{1}{m}\sum_{i=1}^{m}\sum_{t=0}^{T-1} \e_i \ip{v_{i,t}}{ \pi(\flow_t^{\pi_d}(\xi_i)) }  \right] - \frac{1}{m}\sum_{i=1}^{m}\sum_{t=0}^{T-1} \e_i\ip{v_{i,t}}{\pi_g(x_t^{\pi_d}(\xi_t))} \right) \\
    &\stackrel{(a)}{=} \sup_{\{v_{i,t}\} \subset \mathbb{S}^{d-1}} \left(\left[ \sup_{\pi \in \Pi} \frac{1}{m}\sum_{i=1}^{m}\sum_{t=0}^{T-1} \e_i \ip{v_{i,t}}{ \pi(\flow_t^{\pi_d}(\xi_i)) }  \right] - \frac{1}{m}\sum_{i=1}^{m}\sum_{t=0}^{T-1} \e_i\ip{v_{i,t}}{\pi_g(x_t^{\pi_d}(\xi_t))} \right) \\
    &= \sup_{\pi \in \Pi} \frac{1}{m}\sum_{i=1}^{m}\e_i \ell_{\pi_d}(\xi_i; \pi, \pi_g).
\end{align*}
We focus on justifying (a). Recall that for
every normed vector space $X$,
bounded subset $E$, and continuous linear functional $f : X \mapsto \R$, we have 
$\sup_{x \in \overline{\conv}(E)} f(x) = \sup_{x \in E} f(x)$,
where $\overline{\conv}(E)$ denotes the closure of the convex hull of $E$.
By Proposition~\ref{prop:inc_gain_stab_compare_ics},
$\norm{\flow_t^{\pi_d}(\xi)}_2 \leq \zeta B_0^{\alpha_0}$ for any
policy $\pi_d \in \calS(a, 1, \Psi)$ and $\norm{\xi}_2 \leq B_0$.
Hence, under the assumption that $\Pi$ is uniformly bounded,
the image of $\{ x_t^{\pi_d}(\xi_i) \}_{i=1,t=0}^{m,T-1}$ under $\Pi$ is a bounded set in $(\R^d)^{\times mT}$.
Therefore, we can apply this fact about linear functionals to the
linear function:
\begin{align*}
    (\R^d)^{\times mT} \ni \{\psi_{i,t}\} \mapsto \frac{1}{m}\sum_{i=1}^{m}\sum_{t=0}^{T-1} \e_i \ip{v_{i,t}}{\psi_{i,t}},
\end{align*}
from which (a) follows.
\end{proof}

\subsection{Proof}

We first state our main meta-theorem, from which we deduce our rates.

\begin{mythm}
\label{thm:main_meta}
Suppose that Assumption~\ref{assumption:main}
and Assumption~\ref{assumption:stability} hold.
Suppose that $E \leq m$ divides $m$. Define $\Gamma(m, E, \delta)$ as:
\begin{align*}
    \Gamma(m,E,\delta) := 2\mathcal{R}_{m/E}(\Pi) + B_\ell\sqrt{\frac{\log(4E/\delta)}{m/E}}.
\end{align*}
Fix a $\delta \in (0, 1)$.
Assume that for all $k \in \{0, \dots, E-2\}$:
\begin{align*}
    c_k + \Gamma(m, E, \delta) \leq 1.
\end{align*}
Suppose that $\alpha \in (0, 1]$ satisfies
\begin{align*}
    \frac{(1-\alpha)^E}{\alpha} \leq 1.
\end{align*}
For $k \in \{0, \dots, E-1\}$, define $\beta_k(m, E, \delta)$ as:
\begin{align*}
        \beta_k(m, E, \delta) := 2\alpha k \Gamma(m, E, \delta) + 8 L_\Delta \gamma T^{1-1/a} \sum_{i=0}^{k-1} 
    \left(\alpha c_i + \alpha \Gamma(m, E, \delta) \right)^{1/a}.
\end{align*}
With probability at least $1-\delta$ (over $\{\xi_i^k\}_{i=1,k=0}^{m/E,E-1}$ drawn i.i.d.~from $\calD$), we have that the following inequalities simultaneously hold
for the policies $\pi_1, \dots, \pi_E$ produced by Algorithm~\ref{alg:csmile}:
\begin{align*}
    \E_{\xi\sim\calD} \ell_{\pi_k}(\xi;\pi_k,\pi_\star) \leq \beta_k(m, E, \delta), \:\: k =1, \dots, E-1,
\end{align*}
and furthermore,
\begin{align*}
    &~~~~\E_{\xi\sim\calD} \ell_{\pi_E}(\xi;\pi_E,\pi_\star) \\
    &\leq\frac{1}{1-(1-\alpha)^E} \beta_{E-1}(m, E, \delta) + \frac{2\alpha}{1-(1-\alpha)^E} \Gamma(m, E, \delta) \\
    &\qquad+ 4 L_\Delta \gamma T^{1-1/a} \left( \frac{2(1-\alpha)^E}{1-(1-\alpha)^E} \beta_{E-1}(m, E, \delta) + \frac{2\alpha}{1-(1-\alpha)^E} \Gamma(m, E, \delta)  \right) \\
    &\qquad+4 L_\Delta \gamma T^{1-1/a} \left( \frac{2(1-\alpha)^E}{1-(1-\alpha)^E} \beta_{E-1}(m, E, \delta) + \frac{2\alpha}{1-(1-\alpha)^E} \Gamma(m, E, \delta)  \right)^{1/a}.
\end{align*}
\end{mythm}
\begin{proof}

We first use induction on $k$ to show that:
\begin{align*}
    \E_{\xi\sim\calD} \ell_{\pi_k}(\xi; \pi_k,\pi_\star) \leq \beta_k(m, E, \delta), \:\: k=1,\dots,E-1.
\end{align*}

\paragraph{Base case:}
As $\pi_\star \in \Pi$ 
and $\fcl^{\pi_\star}$ is $(a, 1, \Psi)$-incrementally-gain-stable
by Assumption~\ref{assumption:stability}, we have that the optimization problem defining $\hat{\pi}_0$ is feasible.
By Proposition~\ref{prop:generalization_bounds},
there exists an event $\calE_0$ such that $\Pr(\calE_0) \geq 1 - \delta/E$ and on $\calE_0$,
\begin{align*}
    \E_{\xi \sim \calD} \ell_{\pi_0}(\xi; \hat{\pi}_0, \pi_\star) &\stackrel{(a)}{\leq} \frac{1}{m/E}\sum_{i=1}^{m/E}\ell_{\pi_0}(\xi_i^0; \hat{\pi}_0, \pi_\star) + \Gamma(m, E, 2\delta)
    \\
    &\stackrel{(b)}{\leq} \frac{1}{m/E} \sum_{i=1}^{m/E}\ell_{\pi_0}(\xi_i^0; \pi_\star, \pi_\star) + \Gamma(m, E, 2\delta)
    \\
    &\stackrel{(c)}{=} \Gamma(m, E, 2\delta),
\end{align*}
where (a) follows from Proposition~\ref{prop:generalization_bounds} on event $\calE_0$, (b) by feasibility of $\pi_\star=\pi_0$ and optimality of $\hat\pi_0$ to optimization problem \eqref{eq:csmile_opt}, and  (c) from $\ell_{\pi'}(\xi; \pi, \pi) = 0$ for all $\pi',\pi,\xi$.  This sequence of arguments will be used repeatedly in the sequel.

Our goal is to bound $\E_{\xi \sim \calD} \ell_{\pi_1}(\xi; \pi_1, \pi_\star)$.
We observe:
\begin{align*}
    \Delta_{\pi_1, \pi_\star}(x) &= \pi_1(x) - \pi_\star(x) \\
    &=  (1-\alpha) \pi_\star(x) + \alpha \hat{\pi}_0(x) - \pi_\star(x) \\
    &= \alpha  (\hat{\pi}_0(x) - \pi_\star(x)) \\
    &= \alpha \Delta_{\hat{\pi}_0, \pi_\star}(x).
\end{align*}
Therefore, by our assumption that $\Delta_{\hat{\pi}_0,\pi_\star}$ is $L_\Delta$-Lipschitz:
\begin{align*}
    \E_{\xi \sim \calD} \ell_{\pi_1}(\xi; \pi_1, \pi_\star) &= \E_{\xi \sim \calD} \sum_{t=0}^{T-1} \norm{\Delta_{\pi_1, \pi_\star}(\flow_t^{\pi_1}(\xi))}_2 = \alpha \E_{\xi \sim \calD} \sum_{t=0}^{T-1} \norm{\Delta_{\hat{\pi}_0, \pi_\star}(\flow_t^{\pi_1}(\xi))}_2 \\
    &\leq \alpha \E_{\xi \sim \calD} \sum_{t=0}^{T-1} \norm{\Delta_{\hat{\pi}_0, \pi_\star}(\flow_t^{\pi_0}(\xi))}_2 + \alpha L_\Delta \E_{\xi \sim \calD} \sum_{t=0}^{T-1} \norm{\flow_t^{\pi_1}(\xi) - \flow_t^{\pi_0}(\xi)}_2 \\
    &= \alpha \E_{\xi \sim \calD} \ell_{\pi_0}(\xi; \hat{\pi}_0, \pi_\star) + \alpha L_\Delta \E_{\xi \sim \calD} \sum_{t=0}^{T-1} \norm{\flow_t^{\pi_1}(\xi) - \flow_t^{\pi_0}(\xi)}_2.
\end{align*}
Now using the observation that $\Delta_{\pi_1,\pi_\star}(x) = \alpha \Delta_{\hat{\pi}_0,\pi_\star}(x)$, we write:
\begin{align*}
    \fcl^{\pi_\star}(x, 0) = f(x, \pi_\star(x)) = f(x, \pi_1(x) + \Delta_{\pi_\star,\pi_1}(x)) = \fcl^{\pi_1}(x, -\alpha \Delta_{\hat{\pi}_0,\pi_\star}(x)).
\end{align*}
Since $\fcl^{\pi_1}$ is $(a, 1, \Psi)$-incrementally-gain-stable as a result of constraint \eqref{eq:inc_gain_stability_constraint}, by Proposition~\ref{prop:inc_gain_stab_compare_inputs}, we have that for all $\xi$:
\begin{align*}
    &~~~~\sum_{t=0}^{T-1} \norm{\flow_t^{\pi_1}(\xi) - \flow_t^{\pi_0}(\xi)}_2 \\
    &\leq 4 \gamma T^{1-1/a} \max\left\{
   \alpha \sum_{t=0}^{T-2} \norm{\Delta_{\hat{\pi}_0, \pi_\star}(\flow_t^{\pi_0}(\xi))}_2 , \left(\alpha \sum_{t=0}^{T-2} \norm{\Delta_{\hat{\pi}_0, \pi_\star}(\flow_t^{\pi_0}(\xi))}_2 \right)^{1/a}
    \right\} \\
    &\leq 4 \gamma T^{1-1/a} \left[
   \alpha \ell_{\pi_0}(\xi; \hat{\pi}_0, \pi_\star) + \left(\alpha  \ell_{\pi_0}(\xi; \hat{\pi}_0, \pi_\star) \right)^{1/a}
    \right].
\end{align*}
Therefore, on the event $\calE_0$,
\begin{align*}
    &~~~~\E_{\xi \sim \calD} \ell_{\pi_1}(\xi; \pi_1, \pi_\star) \\
    &\leq \alpha \Gamma(m, E, 2\delta) + 4\alpha L_\Delta \gamma T^{1-1/a} \left[
    \alpha \Gamma(m, E, 2\delta) + \left(\alpha \Gamma(m, E, 2\delta) \right)^{1/a}
    \right] \\
    &\leq \alpha \Gamma(m, E, 2\delta) + 8\alpha L_\Delta \gamma T^{1-1/a} 
    \left(\alpha \Gamma(m, E, 2\delta) \right)^{1/a} \\
    &\leq \beta_1(n, E, \delta).
\end{align*}
The first inequality above uses Jensen's inequality to move the expectation inside $x \mapsto x^{1/a}$.

\paragraph{Induction step:} We now assume that $k \in \{1, \dots, E-2\}$ and that the event $\calE_{0:k-1} := \bigcap_{j=0}^{k-1} \calE_j$ holds.
The optimization defining $\hat{\pi}_k$ is feasible on $\calE_{0:k-1}$,
since $\pi_k$ satisfies $\ell_{\pi_k}(\xi_i^k; \pi_k, \pi_k) = 0$ for $i=1, \dots, m/E$,
and $\fcl^{\pi_k}$ is $(a, 1, \Psi)$-incrementally-gain-stable
by constraint \eqref{eq:inc_gain_stability_constraint}. By
the inductive hypothesis, we have that
\begin{align}
    \E_{\xi \sim \calD} \ell_{\pi_k}(\xi; \pi_k, \pi_\star) \leq \beta_k(m, E, \delta). \label{eq:inductive_inequality}
\end{align}
By Proposition~\ref{prop:generalization_bounds}
and taking a union bound over $\pi_k, \pi_\star$,
there exists an event $\calE_k$ with $\Pr(\calE_k) \geq 1-\delta/E$
such that on $\calE_k$, the following statement holds:
\begin{align}
    \max_{\pi_t \in \{\pi_k, \pi_\star\}} \sup_{\pi \in \Pi} \bigabs{\E_{\xi \sim \calD} \ell_{\pi_k}(\xi; \pi, \pi_t) - \frac{1}{m/E} \sum_{i=1}^{m/E} \ell_{\pi_k}(\xi_i^k; \pi, \pi_t)} &\leq \Gamma(m, E, \delta).
    \label{eq:uniform_inductive_step}
\end{align}
Furthermore we note that on $\calE_k$, it holds that:
\begin{align}
    \frac{1}{m/E} \sum_{i=1}^{m/E} \ell_{\pi_k}(\xi_i^k; \pi_k, \pi_\star) &\stackrel{(a)}{\leq} \E_{\xi \sim \calD} \ell_{\pi_k}(\xi; \pi_k, \pi_\star) + \Gamma(m, E, \delta)
    \notag\\
    &\stackrel{(b)}{\leq} \beta_k(m, E, \delta) + \Gamma(m, E, \delta) .
    \label{eq:ERM-bound}
\end{align}
Above, (a) follows from \eqref{eq:uniform_inductive_step}
and (b) follows from \eqref{eq:inductive_inequality}.

Our remaining task is to show that on $\calE_{0:k}$ we have:
\begin{align*}
    \E_{\xi \sim \calD} \ell_{\pi_{k+1}}(\xi; \pi_{k+1}, \pi_\star) \leq \beta_{k+1}(m, E, \delta) .
\end{align*}
We proceed with a similar argument as in the base case.
We first write:
\begin{align*}
    \Delta_{\pi_{k+1}, \pi_\star}(x) &= \pi_{k+1}(x) - \pi_\star(x) \\
    &= (1-\alpha) (\pi_k(x) - \pi_\star(x)) + \alpha (\hat{\pi}_k(x) - \pi_\star(x)) \\
    &= (1-\alpha) \Delta_{\pi_k, \pi_\star}(x) + \alpha \Delta_{\hat{\pi}_k, \pi_\star}(x) .
\end{align*}
Therefore since by assumption $\Delta_{\pi_{k+1}, \pi_\star}$ is $L_\Delta$-Lipschitz:
\begin{align*}
    \ell_{\pi_{k+1}}(\xi; \pi_{k+1}, \pi_\star) &= \sum_{t=0}^{T-1} \norm{\Delta_{\pi_{k+1}, \pi_\star}(\flow_t^{\pi_{k+1}}(\xi))}_2 \\
    &\leq \sum_{t=0}^{T-1} \norm{\Delta_{\pi_{k+1}, \pi_\star}(\flow_t^{\pi_k}(\xi))}_2 +  L_\Delta \sum_{t=0}^{T-1} \norm{\flow_t^{\pi_{k+1}}(\xi) - \flow_t^{\pi_{k}}(\xi)}_2 \\
    &\leq (1-\alpha) \ell_{\pi_k}(\xi; \pi_k, \pi_\star) + \alpha \ell_{\pi_k}(\xi; \hat{\pi}_k, \pi_\star) + L_\Delta \sum_{t=0}^{T-1} \norm{\flow_t^{\pi_{k+1}}(\xi) - \flow_t^{\pi_{k}}(\xi)}_2 .
\end{align*}
Now we write:
\begin{align*}
    \fcl^{\pi_k}(x, 0) &= f(x, \pi_k(x)) = f(x, \pi_{k+1}(x) + \Delta_{\pi_k, \pi_{k+1}}(x)) \\
    &= \fcl^{\pi_{k+1}}(x, \Delta_{\pi_k,\pi_{k+1}}(x)) = \fcl^{\pi_{k+1}}(x, -\alpha \Delta_{\hat{\pi}_k,\pi_k}(x)),
\end{align*}
and therefore given that $\fcl^{\pi_{k+1}}$ is $(a, 1, \Psi)$-incrementally-gain-stable 
by constraint \eqref{eq:inc_gain_stability_constraint}, by Proposition~\ref{prop:inc_gain_stab_compare_inputs},
we have that for all $\xi$:
\begin{align*}
    \sum_{t=0}^{T-1} \norm{\flow_t^{\pi_{k+1}}(\xi) - \flow_t^{\pi_k}(\xi)}_2 &\leq 4 \gamma T^{1-1/a} \max\left\{
    \alpha \ell_{\pi_k}(\xi; \hat{\pi}_k, \pi_k), \left(\alpha \ell_{\pi_k}(\xi; \hat{\pi}_k, \pi_k) \right)^{1/a} 
    \right\} \\
    &\leq 4 \gamma T^{1-1/a} \left[
    \alpha \ell_{\pi_k}(\xi; \hat{\pi}_k, \pi_k) + \left(\alpha  \ell_{\pi_k}(\xi; \hat{\pi}_k, \pi_k) \right)^{1/a} 
    \right].
\end{align*}
Combining this inequality with the inequality above,
\begin{align*}
    \ell_{\pi_{k+1}}(\xi; \pi_{k+1}, \pi_\star) &\leq (1-\alpha) \ell_{\pi_k}(\xi; \pi_k, \pi_\star) + \alpha \ell_{\pi_k}(\xi; \hat{\pi}_k, \pi_\star) \\
    &\qquad + 4 L_\Delta \gamma T^{1-1/a} \left[
    \alpha \ell_{\pi_k}(\xi; \hat{\pi}_k, \pi_k) + \left(\alpha  \ell_{\pi_k}(\xi; \hat{\pi}_k, \pi_k) \right)^{1/a} 
    \right].
\end{align*}
Taking expectations of both sides, applying \eqref{eq:inductive_inequality}, and using Jensen's inequality, we obtain
\begin{align*}
    \E_{\xi \sim \calD} \ell_{\pi_{k+1}}(\xi; \pi_{k+1}, \pi_\star) &\leq (1-\alpha) \beta_k(m, E, \delta) + \alpha \E_{\xi \sim \calD} \ell_{\pi_k}(\xi; \hat{\pi}_k, \pi_\star) \\
    &\qquad + 4 L_\Delta \gamma T^{1-1/a} \left[
    \alpha \E_{\xi \sim \calD} \ell_{\pi_k}(\xi; \hat{\pi}_k, \pi_k) + \left(\alpha \E_{\xi \sim \calD} \ell_{\pi_k}(\xi; \hat{\pi}_k, \pi_k) \right)^{1/a}
    \right].
\end{align*}
Now on $\calE_k$ we have:
\begin{align*}
     \E_{\xi \sim \calD} \ell_{\pi_k}(\xi; \hat{\pi}_k, \pi_k) &\stackrel{(a)}{\leq} \frac{1}{m/E}\sum_{i=1}^{m/E}\ell_{\pi_k}(\xi_i^k; \hat{\pi}_k, \pi_k) + \Gamma(m, E, \delta)
     \\
     &\stackrel{(b)}{\leq} c_k + \Gamma(m, E, \delta),
\end{align*}
where the first inequality (a) follows from \eqref{eq:uniform_inductive_step}, and the second inequality (b) from $\hat\pi_k$ being feasible to the constrained optimization problem \eqref{eq:csmile_opt}.

Similarly, we have that
\begin{align*}
     \E_{\xi \sim \calD} \ell_{\pi_k}(\xi; \hat{\pi}_k, \pi_\star) &\stackrel{(a)}{\leq} \frac{1}{m/E}\sum_{i=1}^{m/E}\ell_{\pi_k}(\xi_i^k; \hat{\pi}_k, \pi_\star) + \Gamma(m, E, \delta)
     \\
     &\stackrel{(b)}{\leq} \frac{1}{m/E}\sum_{i=1}^{m/E}\ell_{\pi_k}(\xi_i^k; {\pi}_k, \pi_\star) + \Gamma(m, E, \delta)
     \\
     &\stackrel{(c)}{\leq} \E_{\xi\sim{}\calD}\ell_{\pi_k}(\xi; {\pi}_k, \pi_\star) + 2 \Gamma(m, E, \delta)
     \\
     &\stackrel{(d)}{\leq} \beta_k(m, E, \delta) + 2 \Gamma(m, E, \delta),
\end{align*}
where (a) follows from \eqref{eq:uniform_inductive_step}, (b) from using $\pi_k$ as a feasible point for optimization problem \eqref{eq:csmile_opt} and optimality of $\hat\pi_k$, (c) from another application \eqref{eq:uniform_inductive_step}, and 
(d) follows from \eqref{eq:inductive_inequality}.

Hence, we have:
\begin{align*}
     \E_{\xi \sim \calD} \ell_{\pi_{k+1}}(\xi; \pi_{k+1}, \pi_\star) 
     &\leq \beta_k(m, E, \delta) + 2 \alpha \Gamma(m, E, \delta) \\
     &\qquad + 4 L_\Delta \gamma T^{1-1/a} \left[
    \left(\alpha c_k + \alpha \Gamma(m, E, \delta) \right) + \left(\alpha c_k + \alpha \Gamma(m, E, \delta) \right)^{1/a}
    \right] \\
    &\leq \beta_k(m, E, \delta) + 2 \alpha \Gamma(m, E, \delta) + 8 L_\Delta \gamma T^{1-1/a} 
    \left(\alpha c_k + \alpha \Gamma(m, E, \delta) \right)^{1/a} \\
    &= \beta_{k+1}(m, E, \delta).
\end{align*}

This finishes the inductive step.
Thus we conclude that on the event
$\calE_{0:E-2}$, which occurs with probability at least $1- \frac{E-1}{E}\delta$, we have that for $k=1, \dots, E-1$,
\begin{align*}
    \E_{\xi \sim \calD} \ell_{\pi_k}(\xi; \pi_k, \pi_\star) &\leq \beta_k(m, E, \delta).
\end{align*}

\paragraph{Final bound:} We now assume $k=E-1$
and that the event $\calE_{0:E-2}$ holds.
On this event:
\begin{align}
    \E_{\xi \sim \calD} \ell_{\pi_{E-1}}(\xi; \pi_{E-1}, \pi_\star) \leq \beta_{E-1}(m, E, \delta). \label{eq:inductive_inequality_last}
\end{align}
We first check feasiblity of the optimization defining $\hat\pi_{E-1}$.
Define $\tilde{\pi}_{E-1}$ as:
\begin{align*}
    \tilde{\pi}_{E-1} := \frac{(1-\alpha)^E}{\alpha} \pi_\star + \left(1 - \frac{(1-\alpha)^E}{\alpha}\right) \pi_{E-1}.
\end{align*}
By our assumption that $\frac{(1-\alpha)^E}{\alpha} \leq 1$, we have that
$\tilde{\pi}_{E-1} \in \Pi$ by convexity of $\Pi$.
Next, we have:
\begin{align*}
    \Delta_{\tilde{\pi}_{E-1}, \pi_{E-1}}(x) &= \frac{(1-\alpha)^E}{\alpha} \pi_\star(x) + \left(1 - \frac{(1-\alpha)^E}{\alpha}\right) \pi_{E-1}(x) - \pi_{E-1}(x) \\
    &= \frac{(1-\alpha)^E}{\alpha} (\pi_\star(x) - \pi_{E-1}(x)) \\
    &= \frac{(1-\alpha)^E}{\alpha} \Delta_{\pi_\star,\pi_{E-1}}(x).
\end{align*}
Therefore:
\begin{align*}
    \frac{1}{m/E} \sum_{i=1}^{m/E} \ell_{\pi_{E-1}}(\xi_i^{E-1}; \tilde{\pi}_{E-1}, \pi_{E-1}) &= \frac{(1-\alpha)^E}{\alpha} \frac{1}{m/E} \sum_{i=1}^{m/E} \ell_{\pi_{E-1}}(\xi_i^{E-1}; \pi_\star, \pi_{E-1}) = c_{E-1},
\end{align*}
which shows that $\tilde{\pi}_{E-1}$ satisfies constraint \eqref{eq:opt:trust_region} with equality.
Next, we observe that:
\begin{align*}
    \frac{1}{1-(1-\alpha)^E} \left[ (1-\alpha) \pi_{E-1} + \alpha \tilde{\pi}_{E-1} - (1-\alpha)^E \pi_\star \right] = \pi_{E-1},
\end{align*}
and hence $\tilde{\pi}_{E-1}$ satisfies constraint \eqref{eq:inc_gain_stability_constraint}
since $\fcl^{\pi_{E-1}}$ is $(a, 1, \Psi)$-incrementally-gain-stable by constraint \eqref{eq:inc_gain_stability_constraint} from the previous iteration.
This shows the optimization problem defining $\hat\pi_{E-1}$ is feasible.

Now as in the inductive step,
by Proposition~\ref{prop:generalization_bounds}
and taking a union bound over $\pi_{E-1}, \pi_\star$,
there exists an event $\calE_{E-1}$ with $\Pr(\calE_{E-1}) \geq 1-\delta/E$
such that on $\calE_{E-1}$, the following statement holds:
\begin{align}
    \max_{\pi_t \in \{\pi_{E-1}, \pi_\star\}} \sup_{\pi \in \Pi} \bigabs{\E_{\xi \sim \calD} \ell_{\pi_{E-1}}(\xi; \pi, \pi_t) - \frac{1}{m/E} \sum_{i=1}^{m/E} \ell_{\pi_{E-1}}(\xi_i^{E-1}; \pi, \pi_t)} &\leq \Gamma(m, E, \delta).
    \label{eq:uniform_inductive_step_last}
\end{align}
Furthermore we note that on $\calE_{E-1}$, it holds that:
\begin{align}
    \frac{1}{m/E} \sum_{i=1}^{m/E} \ell_{\pi_{E-1}}(\xi_i^{E-1}; \pi_{E-1}, \pi_\star) \leq \beta_{E-1}(m, E, \delta) + \Gamma(m, E, \delta) .
    \label{eq:ERM-bound-last}
\end{align}
Therefore we can bound $c_{E-1}$ on $\calE_{E-1}$ by:
\begin{align}
    c_{E-1} \leq \frac{(1-\alpha)^E}{\alpha} ( \beta_{E-1}(m, E, \delta) + \Gamma(m, E, \delta) ) . \label{eq:C_E_minus_1_bound}
\end{align}
We will use this bound in the sequel.

Now we write:
\begin{align*}
    \Delta_{\pi_E,\pi_\star}(x) &=  \frac{1}{1-(1-\alpha)^E} \left[ (1-\alpha) \pi_{E-1}(x) + \alpha \hat{\pi}_{E-1}(x) - (1-\alpha)^E \pi_\star(x) \right] - \pi_\star(x) \\
    &= \frac{1-\alpha}{1-(1-\alpha)^E} \Delta_{\pi_{E-1},\pi_\star}(x) + \frac{\alpha}{1-(1-\alpha)^E} \Delta_{\hat{\pi}_{E-1},\pi_\star}(x) .
\end{align*}
Therefore since $\Delta_{\pi_E,\pi_\star}$ is $L_\Delta$-Lipschitz by assumption,
\begin{align}
    \ell_{\pi_E}(\xi; \pi_E, \pi_\star) &= \sum_{t=0}^{T-1} \norm{ \Delta_{\pi_E,\pi_\star}(\flow_t^{\pi_E}(\xi)) }_2 \nonumber \\
    &\leq \sum_{t=0}^{T-1} \norm{ \Delta_{\pi_E,\pi_\star}(\flow_t^{\pi_{E-1}}(\xi)) }_2 + L_\Delta \sum_{t=0}^{T-1} \norm{ \flow_t^{\pi_E}(\xi) - \flow_t^{\pi_{E-1}}(\xi)}_2 \nonumber \\
    &\leq \frac{1-\alpha}{1-(1-\alpha)^E} \ell_{\pi_{E-1}}(\xi; \pi_{E-1},\pi_\star) + \frac{\alpha}{1-(1-\alpha)^E} \ell_{\pi_{E-1}}(\xi; \hat\pi_{E-1},\pi_\star) \nonumber \\
    &\qquad + L_\Delta \sum_{t=0}^{T-1} \norm{ \flow_t^{\pi_E}(\xi) - \flow_t^{\pi_{E-1}}(\xi)}_2 . \label{eq:final_loss_decomp}
\end{align}
Now we relate $\fcl^{\pi_E}$ to $\fcl^{\pi_{E-1}}$ in the following manner:
\begin{align*}
    \fcl^{\pi_{E-1}}(x) = f(x, \pi_{E-1}(x)) = f(x, \pi_E(x) + \Delta_{\pi_{E-1},\pi_E}(x)) = \fcl^{\pi_E}(x, -\Delta_{\pi_E,\pi_{E-1}}(x)).
\end{align*}
Furthermore, it is straightforward to check that:
\begin{align*}
    \Delta_{\pi_E,\pi_{E-1}}(x) &= \frac{\alpha}{1-(1-\alpha)^E} \Delta_{\hat{\pi}_{E-1},\pi_{E-1}}(x) + \frac{(1-\alpha)^E}{1-(1-\alpha)^E} \Delta_{\pi_{E-1},\pi_\star}(x).
\end{align*}
From constraint \eqref{eq:inc_gain_stability_constraint}, we have that $\fcl^{\pi_E}$ is $(a, 1, \Psi)$-incrementally-gain-stable
and therefore by Proposition~\ref{prop:inc_gain_stab_compare_inputs},
we have for all $\xi$:
\begin{align*}
    &~~~\sum_{t=0}^{T-1} \norm{ \flow_t^{\pi_E}(\xi) - \flow_t^{\pi_{E-1}}(\xi)}_2 \\
    &\leq 4 \gamma T^{1-1/a} 
    \left( \frac{\alpha}{1-(1-\alpha)^E} \ell_{\pi_{E-1}}(\xi; \hat{\pi}_{E-1}, \pi_{E-1}) + \frac{(1-\alpha)^E}{1-(1-\alpha)^E} \ell_{\pi_{E-1}}(\xi; \pi_{E-1}, \pi_\star)  \right) \\
    &\qquad + 4 \gamma T^{1-1/a} 
    \left( \frac{\alpha}{1-(1-\alpha)^E} \ell_{\pi_{E-1}}(\xi; \hat{\pi}_{E-1}, \pi_{E-1}) + \frac{(1-\alpha)^E}{1-(1-\alpha)^E} \ell_{\pi_{E-1}}(\xi; \pi_{E-1}, \pi_\star)  \right)^{1/a}.
\end{align*}
Combining this inequality with \eqref{eq:final_loss_decomp},
taking expectations
and applying Jensen's inequality:
\begin{align}
    &~~~~\E_{\xi\sim\calD} \ell_{\pi_E}(\xi;\pi_E,\pi_\star) \nonumber \\
    &\leq \frac{1-\alpha}{1-(1-\alpha)^E} \E_{\xi\sim\calD}\ell_{\pi_{E-1}}(\xi; \pi_{E-1},\pi_\star) + \frac{\alpha}{1-(1-\alpha)^E} \E_{\xi\sim\calD}\ell_{\pi_{E-1}}(\xi; \hat\pi_{E-1},\pi_\star) \label{eq:final_bound_starting_point} \\
    &\quad+ 4 L_\Delta \gamma T^{1-1/a} \left( \frac{\alpha}{1-(1-\alpha)^E} \E_{\xi\sim\calD}\ell_{\pi_{E-1}}(\xi; \hat{\pi}_{E-1}, \pi_{E-1}) + \frac{(1-\alpha)^E}{1-(1-\alpha)^E} \E_{\xi\sim\calD}\ell_{\pi_{E-1}}(\xi; \pi_{E-1}, \pi_\star)  \right) \nonumber \\
    &\quad + 4 L_\Delta \gamma T^{1-1/a} \left( \frac{\alpha}{1-(1-\alpha)^E} \E_{\xi\sim\calD}\ell_{\pi_{E-1}}(\xi; \hat{\pi}_{E-1}, \pi_{E-1}) + \frac{(1-\alpha)^E}{1-(1-\alpha)^E} \E_{\xi\sim\calD}\ell_{\pi_{E-1}}(\xi; \pi_{E-1}, \pi_\star)  \right)^{1/a}. \nonumber
\end{align}
Now on $\calE_{E-1}$ we have:
\begin{align*}
    \E_{\xi\sim\calD}\ell_{\pi_{E-1}}(\xi; \hat\pi_{E-1},\pi_{E-1}) &\stackrel{(a)}{\leq} \frac{1}{m/E} \sum_{i=1}^{m/E}\ell_{\pi_{E-1}}(\xi_i^{E-1}; \hat\pi_{E-1},\pi_{E-1}) + \Gamma(m, E, \delta) \\
    &\stackrel{(b)}{\leq} c_{E-1} + \Gamma(m, E, \delta) \\
    &\stackrel{(c)}{\leq} \frac{(1-\alpha)^E}{\alpha} (  \beta_{E-1}(m, E, \delta) + \Gamma(m, E, \delta) ) + \Gamma(m, E, \delta) \\
    &= \frac{(1-\alpha)^E}{\alpha} \beta_{E-1}(m, E, \delta) + \left(1 + \frac{(1-\alpha)^E}{\alpha} \right) \Gamma(m, E, \delta) \\
    &\leq \frac{(1-\alpha)^E}{\alpha} \beta_{E-1}(m, E, \delta) + 2 \Gamma(m, E, \delta).
\end{align*}
Here, (a) follows by \eqref{eq:uniform_inductive_step_last},
(b) follows from constraint \eqref{eq:opt:trust_region},
and (c) follows from \eqref{eq:C_E_minus_1_bound}.
Recalling that
$\E_{\xi\sim\calD}\ell_{\pi_{E-1}}(\xi; \pi_{E-1}, \pi_\star) \leq \beta_{E-1}(m, E, \delta)$, the previous inequality yields:
\begin{align}
    &\frac{\alpha}{1-(1-\alpha)^E} \E_{\xi\sim\calD}\ell_{\pi_{E-1}}(\xi; \hat{\pi}_{E-1}, \pi_{E-1}) + \frac{(1-\alpha)^E}{1-(1-\alpha)^E} \E_{\xi\sim\calD}\ell_{\pi_{E-1}}(\xi; \pi_{E-1}, \pi_\star) \nonumber \\
    &\qquad\leq \frac{2(1-\alpha)^E}{1-(1-\alpha)^E} \beta_{E-1}(m, E, \delta) + \frac{2\alpha}{1-(1-\alpha)^E} \Gamma(m, E, \delta). \label{eq:final_bound_term_two}
\end{align}
Next, observe that:
\begin{align*}
    \Delta_{\tilde{\pi}_{E-1},\pi_\star}(x) &= g(x) \left[ \frac{(1-\alpha)^E}{\alpha} \pi_\star(x) + \left(1 - \frac{(1-\alpha)^E}{\alpha}\right) \pi_{E-1}(x) - \pi_\star(x) \right] \\
    &= \left(1 - \frac{(1-\alpha)^E}{\alpha}\right) g(x) ( \pi_{E-1}(x) - \pi_\star(x) ) \\
    &= \left(1 - \frac{(1-\alpha)^E}{\alpha}\right) \Delta_{\pi_{E-1},\pi_\star}(x) .
\end{align*}
Therefore, for any $\pi,\xi$, we have:
\begin{align}
 \ell_{\pi}(\xi; \tilde{\pi}_{E-1}, \pi_\star) = \left(1 - \frac{(1-\alpha)^E}{\alpha}\right)  \ell_{\pi}(\xi; \pi_{E-1}, \pi_\star) .  \label{eq:loss_tilde_to_no_tilde}
\end{align}
Hence on $\calE_{E-1}$, we have:
\begin{align*}
    \E_{\xi\sim\calD}\ell_{\pi_{E-1}}(\xi; \hat\pi_{E-1},\pi_\star) &\stackrel{(a)}{\leq} \frac{1}{m/E} \sum_{i=1}^{m/E} \ell_{\pi_{E-1}}(\xi_i^{E-1}; \hat\pi_{E-1},\pi_\star) + \Gamma(m, E, \delta) \\
    &\stackrel{(b)}{\leq}  \frac{1}{m/E} \sum_{i=1}^{m/E} \ell_{\pi_{E-1}}(\xi_i^{E-1}; \tilde{\pi}_{E-1}, \pi_\star) + \Gamma(m, E, \delta) \\
    &\stackrel{(c)}{=} \left(1 - \frac{(1-\alpha)^E}{\alpha} \right) \frac{1}{m/E} \sum_{i=1}^{m/E} \ell_{\pi_{E-1}}(\xi_i^{E-1}; \pi_{E-1}, \pi_\star) + \Gamma(m, E, \delta) \\
    &\stackrel{(d)}{\leq} \left(1 - \frac{(1-\alpha)^E}{\alpha} \right) ( \E_{\xi\sim\calD} \ell_{\pi_{E-1}}(\xi; \pi_{E-1}, \pi_\star) + \Gamma(m, E, \delta)) + \Gamma(m, E, \delta) \\
    &\stackrel{(e)}{\leq}  \beta_{E-1}(m, E, \delta) + 2\Gamma(m, E, \delta) .
\end{align*}
Here, (a) follows by \eqref{eq:uniform_inductive_step_last},
(b) follows from the optimality of $\hat\pi_{E-1}$ and feasibility of $\tilde{\pi}_{E-1}$,
(c) follows from \eqref{eq:loss_tilde_to_no_tilde},
(d) follows from another application of \eqref{eq:uniform_inductive_step_last},
and (e) follows from \eqref{eq:inductive_inequality_last}.
This allows us to bound:
\begin{align}
    &\frac{1-\alpha}{1-(1-\alpha)^E} \E_{\xi\sim\calD}\ell_{\pi_{E-1}}(\xi; \pi_{E-1},\pi_\star) + \frac{\alpha}{1-(1-\alpha)^E} \E_{\xi\sim\calD}\ell_{\pi_{E-1}}(\xi; \hat\pi_{E-1},\pi_\star) \nonumber \\
    &\qquad\leq \frac{1}{1-(1-\alpha)^E} \beta_{E-1}(m, E, \delta) + \frac{2\alpha}{1-(1-\alpha)^E} \Gamma(m, E, \delta). \label{eq:final_bound_term_one}
\end{align}
Combining \eqref{eq:final_bound_starting_point},
\eqref{eq:final_bound_term_two}, and \eqref{eq:final_bound_term_one}:
\begin{align*}
    &~~~~\E_{\xi\sim\calD} \ell_{\pi_E}(\xi;\pi_E,\pi_\star)\\
    &\leq\frac{1}{1-(1-\alpha)^E} \beta_{E-1}(m, E, \delta) + \frac{2\alpha}{1-(1-\alpha)^E} \Gamma(m, E, \delta) \\
    &\qquad+ 4 L_\Delta \gamma T^{1-1/a} \left( \frac{2(1-\alpha)^E}{1-(1-\alpha)^E} \beta_{E-1}(m, E, \delta) + \frac{2\alpha}{1-(1-\alpha)^E} \Gamma(m, E, \delta)  \right).
\end{align*}
\end{proof}

With Theorem~\ref{thm:main_meta} in place,
we now turn to the proof of our main results, which are immediate consequences
of Theorem~\ref{thm:main_meta}.
We first restate and prove Theorem~\ref{thm:main_bc}.
\mainbc*
\begin{proof}
Theorem~\ref{thm:main_meta} states that if $\Gamma(m, 1, \delta) \leq 1$, then:
\begin{align*}
    \E_{\xi\sim\calD} \ell_{\pi_1}(\xi;\pi_1,\pi_\star) &\leq \Gamma(m, 1, \delta) + 8 L_\Delta \gamma T^{1-1/a} \Gamma(m, 1, \delta)^{1/a} \\
    &\leq (1 + 8 L_\Delta \gamma T^{1-1/a}) \Gamma(m,1,\delta)^{1/a} \\
    &\lesssim L_\Delta \gamma T^{1-1/a}\Gamma(m,1,\delta)^{1/a}.
\end{align*}
To complete the proof we simply need to bound $\Gamma(m, 1, \delta)$, which has the form:
\begin{align*}
    \Gamma(m, 1, \delta) = 2 \calR_m(\Pi) + B_\ell \sqrt{\frac{\log(4/\delta)}{m}}.
\end{align*}
From Proposition~\ref{prop:almost_sure_bound} and Proposition~\ref{prop:rademacher_bound}, we have that:
\begin{align*}
    B_\ell &\leq 2 \zeta B_0^{\alpha_0} L_\Delta T^{1-1/a}, \\
    \calR_m(\Pi) &\leq  65 \zeta B_0^{\alpha_0}  B_\theta L_{\partial^2 \pi} T^{1-1/a} \sqrt{\frac{q}{m}}.  
\end{align*}
This gives us a bound:
\begin{align*}
    \Gamma(m, 1, \delta) \lesssim \zeta B_0^{\alpha_0}  B_\theta L_{\partial^2 \pi} T^{1-1/a} \left(\frac{q}{m}\right)^{1/2} + \zeta B_0^{\alpha_0} L_\Delta T^{1-1/a} \left( \frac{\log(1/\delta)}{m} \right)^{1/2}.
\end{align*}
Plugging $\delta = e^{-q}$ yields:
\begin{align*}
    \Gamma(m,1,\delta) &\lesssim \zeta B_0^{\alpha_0} T^{1-1/a} \max\{B_\theta L_{\partial^2 \pi}, L_\Delta\} \left(\frac{q}{m}\right)^{1/2} \\
    &= \zeta B_0^{\alpha_0} \bar{L} T^{1-1/a}\left(\frac{q}{m}\right)^{1/2} .
\end{align*}
The claim now follows.
\end{proof}

We now restate and prove Theorem~\ref{thm:main_shift}.
\mainshift*
\begin{proof}
We first bound $\Gamma(m, E, \delta)$, which has the form:
\begin{align*}
    \Gamma(m, E, \delta) = 2\calR_{m/E}(\Pi) + B_\ell \sqrt{\frac{E \log(4E/\delta)}{m}}.
\end{align*}
From Proposition~\ref{prop:almost_sure_bound} and Proposition~\ref{prop:rademacher_bound}, we have that:
\begin{align*}
    B_\ell &\leq 2 \zeta B_0 L_\Delta T^{1-1/a}, \\
    \calR_{m/E}(\Pi) &\leq 65 \zeta B_0^{\alpha_0}  B_\theta L_{\partial^2 \pi} T^{1-1/a} \sqrt{\frac{E q}{m}}.  
\end{align*}
Setting $\delta = e^{-q}$, this yields the bound:
\begin{align}
    \Gamma(m, E, \delta) \lesssim \zeta B_0^{\alpha_0} T^{1-1/a} \bar{L} \sqrt{\frac{E (q \vee \log{E})}{m}}. \label{eq:igs_cmile_proof_bound_zero}
\end{align}
We choose $m$ and $c_k$ such that
(a) $\Gamma(m, E, \delta) \leq 1/2$ and
(b) $c_k \leq \Gamma(m, E, \delta)$ for $k \in \{1, \dots, E-2\}$,
which leads to the constraint \eqref{eq:main_shift_m_req} for (a)
and the constraint \eqref{eq:main_shift_c_req} for (b).

In preparation to apply Theorem~\ref{thm:main_meta},
we use our assumptions to show:
\begin{align}
    \frac{2(1-\alpha)^E}{1-(1-\alpha)^E} \beta_{E-1}(m, E, \delta) + \frac{2\alpha}{1-(1-\alpha)^E} \Gamma(m, E, \delta) \lesssim E \Gamma(m, E, \delta)^{1/a}. \label{eq:igs_cmile_proof_bound_one}
\end{align}
Since $E \geq \frac{1}{\alpha} \log(\frac{1}{\alpha})$, then
$\frac{(1-\alpha)^E}{\alpha} \leq \exp( - \alpha E) /\alpha \leq 1$.
Furthermore, since we also assume that
$\alpha \leq \min\left\{\frac{1}{2}, \frac{1}{L_\Delta \gamma T^{1-1/a}}\right\}$, then
\begin{align*}
    (1-\alpha)^E \leq \min\left\{\frac{1}{2}, \frac{1}{L_\Delta \gamma T^{1-1/a}}\right\}.
\end{align*}
Now we proceed to bound $\beta_{E-1}(m, E, \delta)$:
\begin{align}
    \beta_{E-1}(m, E, \delta) &= 2 \alpha (E-1) \Gamma(m, E, \delta) + 8 L_\Delta \gamma T^{1-1/a} \sum_{i=0}^{E-2}(\alpha c_i + \alpha \Gamma(m, E, \delta))^{1/a} \nonumber \\
    &\lesssim E \Gamma(m, E, \delta) +  L_\Delta \gamma T^{1-1/a}  E \Gamma(m, E, \delta)^{1/a} \nonumber \\
    &\lesssim L_\Delta \gamma T^{1-1/a} E \Gamma(m, E, \delta)^{1/a}. \label{eq:igs_cmile_proof_bound_two}
\end{align}
Combining this bound with the inequalities for $(1-\alpha)^E$ yields
\eqref{eq:igs_cmile_proof_bound_one}.

We now apply Theorem~\ref{thm:main_meta} with \eqref{eq:igs_cmile_proof_bound_one} and
\eqref{eq:igs_cmile_proof_bound_two}:
\begin{align*}
    &~~~~\E_{\xi\sim\calD} \ell_{\pi_E}(\xi;\pi_E,\pi_\star)\\
    &\leq \frac{1}{1-(1-\alpha)^E} \beta_{E-1}(m, E, \delta) + \frac{2\alpha}{1-(1-\alpha)^E} \Gamma(m, E, \delta) \\
    &\qquad+ 4 L_\Delta \gamma T^{1-1/a} \left( \frac{2(1-\alpha)^E}{1-(1-\alpha)^E} \beta_{E-1}(m, E, \delta) + \frac{2\alpha}{1-(1-\alpha)^E} \Gamma(m, E, \delta)  \right) \\
    &\qquad+4 L_\Delta \gamma T^{1-1/a} \left( \frac{2(1-\alpha)^E}{1-(1-\alpha)^E} \beta_{E-1}(m, E, \delta) + \frac{2\alpha}{1-(1-\alpha)^E} \Gamma(m, E, \delta)  \right)^{1/a} \\
    &\lesssim \beta_{E-1}(m, E, \delta) + \Gamma(m, E, \delta) + L_\Delta \gamma T^{1-1/a} \left( E \Gamma(m, E, \delta)^{1/a} + E^{1/a} \Gamma(m, E, \delta)^{1/a^2} \right) \\
    &\lesssim L_\Delta \gamma T^{1-1/a} \left( E \Gamma(m, E, \delta)^{1/a} + E^{1/a} \Gamma(m, E, \delta)^{1/a^2} \right).
\end{align*}
Using the estimate \eqref{eq:igs_cmile_proof_bound_zero},
\begin{align*}
    E^{1/a}\Gamma(m, E, \delta)^{1/a^2} &\lesssim (\zeta B_0^{\alpha_0} \bar{L})^{1/a^2} T^{\frac{1}{a^2} (1-\frac{1}{a})} \left( \frac{E^{2a+1} (q \vee \log{E})}{m} \right)^{\tfrac{1}{2a^2}}.
\end{align*}
Since $E \Gamma(m, E, \delta)^{1/a} \leq 1$ by \eqref{eq:main_shift_m_req},
we have:
\begin{align*}
    \E_{\xi\sim\calD} \ell_{\pi_E}(\xi;\pi_E,\pi_\star) &\lesssim L_\Delta \gamma T^{1-1/a} (\zeta B_0^{\alpha_0} \bar{L})^{1/a^2} T^{\frac{1}{a^2} (1-\frac{1}{a})} \left( \frac{E^{2a+1} (q \vee \log{E})}{m} \right)^{\tfrac{1}{2a^2}} \\
    &= L_\Delta \gamma (\zeta B_0^{\alpha_0} \bar{L})^{1/a^2} \cdot T^{(1-1/a)(1+1/a^2)} \cdot \left( \frac{E^{2a+1} (q \vee \log{E})}{m} \right)^{\tfrac{1}{2a^2}}.
\end{align*}
\end{proof}